\documentclass{article}

\usepackage{times}
\newif\ifshowcomments
\newif\ifciteformat

\showcommentstrue

\ifshowcomments

\newcommand {\michael}[1]{{\color{red}\sf{[Michael: #1]}}}
\newcommand {\addressed}[1]{{\color{blue}\sf{[Addressed: #1]}}}
\newcommand {\ryan}[1]{{\color{purple}\sf{[Ryan: #1]}}}
\newcommand {\lsh}[1]{{\color{green}\sf{[Liam: #1]}}}
\newcommand {\remove}[1]{{\color{yellow}{#1}}}
\newcommand{\yaoqing}[1]{\textcolor{orange}{\sf[Yaoqing:\ #1]}}

\else

\newcommand {\michael}[1]{}
\newcommand {\addressed}[1]{}
\newcommand {\ryan}[1]{}
\newcommand {\lsh}[1]{}
\newcommand {\remove}[1]{}
\newcommand{\yaoqing}[1]{}

\fi

\usepackage{fullpage}
\usepackage[utf8]{inputenc} % allow utf-8 input
\usepackage[T1]{fontenc}    % use 8-bit T1 fonts
\usepackage{url}            % simple URL typesetting
\usepackage{booktabs}       % professional-quality tables
\usepackage{amsfonts}       % blackboard math symbols
\usepackage{nicefrac}       % compact symbols for 1/2, etc.
\usepackage{microtype}      % microtypography
\usepackage{xcolor}
\usepackage[parfill]{parskip}
\usepackage{tikz}
\usepackage{amssymb,amsmath,graphicx,multicol,bm,graphics, bbm}
\usepackage[margin=1.1in]{geometry}
\usepackage[mathscr]{euscript}
\usepackage{ulem}
\newcommand\E{\mathbb{E}}

\newcommand\R{\mathbb{R}}
\renewcommand\P{\mathbb{P}}

\newcommand{\xb}{\boldsymbol{x}}

\newcommand{\zb}{\boldsymbol{z}}

\newcommand{\hb}{\boldsymbol{h}}

\newcommand{\EI}{\mathrm{EIR}}
\newcommand{\DER}{\mathrm{DER}}
\newcommand{\defn}[1]{\textbf{\emph{#1}}}
\newcommand{\mv}{\mathrm{MV}}

\newcommand\dis{D}

\newcommand\testerr{L}

\newcommand{\Qc}{\rho}
\newcommand{\Dc}{\mathcal{D}}

\usepackage[english]{babel}
\usepackage{amsfonts}
\usepackage{amsthm}
\usepackage{mathtools,bbold}
\usepackage{graphics}
\usepackage{fancyhdr}
\usepackage{MnSymbol}
\usepackage{tcolorbox}

\DeclareMathOperator*{\argmax}{arg\,max}

\usepackage[font=small]{caption}
\usepackage{subcaption}
\usepackage{enumerate}
\usepackage{wrapfig, blindtext}
\usepackage{hyperref}       % hyperlinks
\hypersetup{
	colorlinks,
	linkcolor={red!40!gray},
	citecolor={blue!40!gray},
	urlcolor={blue!70!gray}
}

\newtheorem{theorem}{Theorem}
\newtheorem{proposition}{Proposition}

\newtheorem{lemma}{Lemma}

\theoremstyle{definition}
\newtheorem{assumption}{Assumption}
\newtheorem{definition}{Definition}
\newtheorem{example}{Example}
\newtheorem{remark}{Remark}

\title{
When are ensembles really effective?}

  \author{%
          Ryan Theisen \\
  Department of Statistics\\
  University of California, Berkeley\\
  \and
  Hyunsuk Kim \\
  Department of Statistics\\
  University of California, Berkeley\\
  \and
  Yaoqing Yang \\
  Department of Computer Science \\
  Dartmouth College \\
  \and
  Liam Hodgkinson \\
  School of Mathematics and Statistics \\
  University of Melbourne, Australia \\
  \and
   Michael W. Mahoney\\
  International Computer Science Institute \\ 
  Lawrence Berkeley National Laboratory \\ 
  and Department of Statistics\\
  University of California, Berkeley\\
 }

\begin{document}

\maketitle

\begin{abstract}
Ensembling has a long history in statistical data analysis, with many impactful applications. 
However, in many modern machine learning settings, the benefits of ensembling are less ubiquitous and less obvious. 
We study, both theoretically and empirically, the fundamental question of when ensembling yields significant performance improvements in classification tasks. 
Theoretically, we prove new results relating the \emph{ensemble improvement rate} (a measure of how much ensembling decreases the error rate versus a single model, on a relative scale) to the \emph{disagreement-error ratio}. 
We show that ensembling improves performance significantly whenever the disagreement rate is large relative to the average error rate; and that, conversely, one classifier is often enough whenever the disagreement rate is low relative to the average error rate. 
On the way to proving these results, we derive, under a mild condition called \emph{competence}, improved upper and lower bounds on the average test error rate of the majority vote classifier.
To complement this theory, we study ensembling empirically in a variety of settings, verifying the predictions made by our theory, and identifying practical scenarios where ensembling does and does not result in large performance improvements. 
Perhaps most notably, we demonstrate a distinct difference in behavior between interpolating models (popular in current practice) and non-interpolating models (such as tree-based methods, where ensembling is popular), demonstrating that ensembling helps considerably more in the latter case than in the former.
\end{abstract}

\section{Introduction}
% Ensembling, which refers to using methods such as bagging, stacking, boosting, and voting that aim to combine the predictions from multiple base models, is common throughout statistical data analysis.
%
The fundamental ideas underlying ensemble methods can be traced back at least two centuries, with Condorcet's Jury Theorem among its earliest developments \cite{condorcet1785essay}. 
This result asserts that if each juror on a jury makes a correct decision independently and with the same probability $p > 1/2$, then the majority decision of the jury is more likely to be correct with each additional juror. 
The general principle of aggregating knowledge across imperfectly correlated sources is intuitive, and it has motivated many ensemble methods used in modern statistics and machine learning practice. 
Among these, tree-based methods like random forests \cite{breiman2001random} and XGBoost \cite{chen2016xgboost} are some of the most effective and widely-used. 

%%\addressed{Some of the text in the intro of the 3/5/23 version seems a bit better than in this version, e.g., next two pars.  Let's sync on how to position this.}
%\michael{Q: is this commend unaddressed?}

With the growing popularity of deep learning, a number of approaches have been proposed for ensembling neural networks. 
Perhaps the simplest of them are so-called deep ensembles, which are ensembles of neural networks trained from independent initializations \cite{abe2022deep, abebest, fort2019deep}. 
In some cases, it has been claimed that such deep ensembles provide significant improvement in performance \cite{fort2019deep, ovadia2019can, ashukha2020pitfalls}. 
Such ensembles have also been used to obtain uncertainty estimates for prediction \cite{lakshminarayanan2017simple} and to provide more robust predictions under distributional shift. %%\ryan{add cite}. 
However, the benefits of deep ensembling are not universally accepted. 
Indeed, other works have found that ensembling is less necessary for larger models, and that in some cases a single larger model can perform as well as an ensemble of smaller models \cite{hinton2015distilling,bucilua2006model,geiger2020scaling, abe2022deep}. 
Similarly mixed results, where empirical performance does not conform with intuitions and popular theoretical expectations, have been reported in the Bayesian approach to deep learning~\cite{izmailov2021dangers}. Furthermore, an often-cited practical issue with ensembling, especially of large neural networks, is the constraint of storing and performing inference with many distinct models. 

In light of the increase in computational cost, it is of great value to understand exactly when we might expect ensembling to improve performance non-trivially. 
% Given the mixed results reported from ensembling, particularly when used with modern deep neural networks, it is important to understand exactly when one can expect ensembling to be worthwhile. 
In particular, consider the following practical scenario: a practitioner has trained a single (perhaps large and expensive) model, and would like to know whether they can expect significant gains in performance from training additional models and ensembling them. 
This question lacks a sufficient answer, both from the theoretical and empirical perspectives, and hence motivates the main question of the present work: 

\begin{center}
\begin{tcolorbox}[hbox]
When are ensembles \emph{\textbf{really}} effective?
\end{tcolorbox}
\end{center}

The present work addresses this question, both theoretically and empirically, under very general conditions.
We focus our study on the most popular ensemble classifier---the \emph{majority vote classifier} (Definition \ref{def:mv-classifier}), which we denote by $h_\mv$---although our framework also covers variants such weighted majority vote methods. 

\paragraph{Theoretical results.} 
Our main theoretical contributions, contained in Section \ref{sec:theory}, are as follows. 
First, we formally define the \emph{ensemble improvement rate} ($\EI$, Definition \ref{def:eir}), which measures the decrease in error rate from ensembling, on a relative scale. 
We then introduce a new condition called \emph{competence} (Assumption \ref{assumption-1}) that rules out pathological cases, allowing us to prove stronger bounds on the ensemble improvement rate. 
Specifically, 1) we prove (in Theorem \ref{thm:Always}) that competent ensembles can never hurt performance, and 2) we prove (in Theorem \ref{thm:eir-der-linear}) that the $\EI$ can be upper and lower bounded by linear functions of the \emph{disagreement-error ratio} ($\DER$, Definition \ref{def:der}). 
\defn{Our theoretical results predict that ensemble improvement will be high whenever disagreement is large relative to the average error rate (i.e., $\DER > 1$).} 
Moreover, we show (in Appendix \ref{app:mv_bounds}) that as Corollaries of our theoretical results, we obtain new bounds on the error rate of the majority vote classifier that significantly improve on previous results, provided the competence assumption is satisfied.

\paragraph{Empirical results.} 
In light of our new theoretical understanding of ensembling, we perform a detailed empirical analysis of ensembling in practice. 
In Section \ref{sec:evaluating-theory}, we evaluate the assumptions and predictions made by the theory presented in Section \ref{sec:theory}. 
In particular, we verify on a variety of tasks that the competence condition holds, we verify empirically the linear relationship between the $\EI$ and the $\DER$, as predicted by our bounds, and we suggest directions through which our theoretical results might be improved. 
In Section \ref{sec:interpolation}, we provide significant evidence for distinct behavior arising for ensembles in and out of the ``interpolating regime,'' i.e., when each of the constituent classifiers in an ensemble has sufficient capacity to achieve zero training error. 
    \defn{We show 
    1) that interpolating ensembles exhibit consistently lower ensemble improvement rates, and 
    2) that this corresponds to ensembles transitioning (sometimes sharply) from the regime $\DER > 1$ to $\DER < 1$.} 
    Finally, we also show that tree-based ensembles represent a unique exception to this phenomenon, making them particularly well-suited to ensembling.

In addition to the results presented in the main text, we provide supplemental theoretical results (including all proofs) in Appendix \ref{app:proofs}, as well as supplemental empirical results in Appendix~\ref{app:additional-empirical}.
\section{Background and preliminaries}
\label{sec:background}

In this section, we present some background as well as preliminary results of independent interest that set the context for our main theoretical and empirical results.

\subsection{Setup}
\label{sec:setup}

In this work, we focus on the $K$-class classification setting, wherein the data $(X,Y) \in \mathcal{X}\times \mathcal{Y} \sim \Dc$ consist of features $\xb \in \mathcal{X}$ and labels $y\in \mathcal{Y} = \{1,\dots, K\}$. 
Classifiers are then functions $h:\mathcal{X} \to \mathcal{Y}$ that belong to some set $\mathcal{H}$. 
To measure the performance of a single classifier $h$ on the data distribution $\Dc$, we use the usual error~rate:
\begin{align*}
\testerr_\Dc(h) = \E_{X,Y\sim \Dc}[\mathbb{1}(h(X)\neq Y)].
\end{align*}
For notational convenience, we drop the explicit dependence on $\Dc$ whenever it is apparent from~context.

A central object in our study is a distribution $\Qc$ over classifiers.
Depending on the context, this distribution could represent a variety of different things. 
For example, $\rho$ could be: 
\begin{enumerate}[i)]
\item A discrete distribution on a finite set of classifiers $\{h_1,\dots,h_M\}$ with weights $\rho_1,\dots,\rho_M$, e.g., representing normalized weights in a weighted ensembling scheme; 
\item A distribution over parameters $\theta$ of a parametric family of models, $h_\theta$, determined, e.g., by a stochastic optimization algorithm with random initialization; 
\item A Bayesian posterior distribution.
\end{enumerate}

The distribution $\rho$ induces two error rates of interest. 
The first is the \defn{average error rate of any single classifier} under $\rho$, defined to be $\E_{h\sim\rho}[\testerr(h)]$. 
The second is the \defn{error rate of the majority vote classifier}, $h_\mv$, which is defined for a distribution $\rho$ as~follows.

\begin{definition}[Majority vote classifier]
\label{def:mv-classifier}
Given $\Qc$, the \defn{majority vote classifier} is the classifier which, for an input $\xb$, predicts the most probable class for this input among classifiers drawn from $\Qc$,
\begin{align*}
h_{\mv}(\xb) = \argmax_{j}\; \E_{h\sim \Qc}[\mathbb{1}(h(\xb)=j)].
\end{align*}
\end{definition}
In the Bayesian context, $\rho = \rho(h\mid X_{\text{train}}, y_{\text{train}})$ is a posterior distribution over classifiers.
In this case, the majority vote classifier is often called the Bayes classifier, and the error rate $L(h_\mv)$ is called the Bayes error rate. 
In such contexts, the average error rate is often referred to as the Gibbs error rate associated with $\rho$ and~$\Dc$. 

Finally, we will present results in terms of the \emph{disagreement rate} between classifiers drawn from a distribution $\rho$, defined as follows.

\begin{definition}[Disagreement]
    The \defn{disagreement rate} between two classifiers $h,h'$ is given by $\dis_{\Dc}(h,h')=\E_{X \sim \Dc}[\mathbb{1}(h(X)\neq h'(X))]$. 
    The \defn{expected disagreement rate} is $\E_{h,h'\sim \rho}[\dis_{\Dc}(h,h')]$, where $h,h'\sim \rho$ are drawn independently.
    \label{def:disagreement}
\end{definition}

\subsection{Prior work}
\label{sec:existing-theory}

\paragraph{Ensembling theory.}
Perhaps the simplest general relation between the majority vote error rate and the average error rate guarantees only that the majority vote classifier is no worse than twice the average error rate \cite{LAVIOLETTE201715, second-order-mv-bounds2020}. 
To see this, let $W_\rho \equiv W_\rho(X,Y) = \mathbb{E}_{h\sim \rho}[\mathbb{1}(h(X) \neq Y)]$ denote the proportion of erroneous classifiers in the ensemble for a randomly sampled input-output pair $(X,Y) \sim \mathcal{D}$, and note that $\E[W_\rho] = \E[L(h)]$. 
Then, by a ``first-order'' application of Markov's inequality, we have that
\begin{align}
\label{eq:trivial-FO}
0 \leq \testerr(h_\mv) \leq \mathbb{P}(W_\rho \geq 1/2) \leq 2\, \mathbb{E}[W_\rho] = 2\,\E_{h\sim \Qc}[\testerr(h)].
\end{align}
This bound is almost always uninformative in practice. Indeed, it may seem surprising that an ensemble classifier could perform \emph{worse} than the average of its constituent classifiers, much less a factor of two worse. 
Nonetheless, the first-order upper bound is, in fact, tight: there exist distributions $\rho$ (over classifiers) and $\Dc$ (over data) such that the majority vote classifier is twice as erroneous as any one classifier, on average. As one might expect, however, this tends to happen only in pathological cases; we give examples of such ensembles in Appendix~\ref{app:pathological}.

To circumvent the shortcomings of the simple first-order bound, more recent approaches have developed bounds incorporating ``second-order'' information from the distribution $\rho$ \cite{second-order-mv-bounds2020}. 
One successful example of this is given by a class of results known as C-bounds \cite{JMLR:v16:germain15a,LAVIOLETTE201715}.
The most general form of these bounds states, provided $\E[M_\rho(X,Y)] > 0$, that $\testerr(h_\mv) \leq 1 -\E[M_\rho(X,Y)]^2/\E[M_\rho^2(X,Y)]$, where $M_\rho(X,Y) = \E_{h\sim \rho}[ \mathbb{1}(h(X) = Y)] - \max_{j\neq Y}\E_{h\sim \rho}[\mathbb{1}(h(X) = j)]$
is called the \emph{margin}. 
In the binary classification case, the condition $\E[M_\rho(X,Y)] > 0$ is equivalent to the assumption $\E_{h\sim \rho}[\testerr(h)] < 1/2$.
Hence, it can be viewed as a requirement that individual classifiers are ``weak learners.'' 
The same condition is used to derive a very similar bound for random forests in \cite{breiman2001random}, which is then further upper bounded to obtain a more intuitive (though weaker) bound in terms of the ``c/s2''~ratio. Relatedly, \cite{second-order-mv-bounds2020} obtains a bound on the error rate of the majority-vote classifier, in the special case of binary classification, directly in terms of the disagreement rate, taking the form $4\E[L(h)] - 2\E[\dis(h,h')]$. We note that our theory improves this bound by factor of $2$ (see Appendix \ref{app:mv_bounds}). Other results obtain similar expressions, but in terms of different loss functions, e.g., cross-entropy loss~\cite{abe2022deep, ortega22a}.

\paragraph{Other related studies.} In addition to theoretical results, there have been a number of recent empirical studies investigating the use of ensembling. Perhaps the most closely related is \cite{abebest}, which shows, perhaps surprisingly, that ensembles do not benefit significantly from encouraging diversity during training. In contrast to the present work, \cite{abebest} focuses on the cross entropy loss for classification (which facilitates somewhat simpler theoretical analysis), whereas we focus on the more intuitive and commonly used classification error rate. Moreover, while \cite{abebest} study the ensemble improvement \emph{gap} (i.e., the difference between the average loss of a single classifier and the ensemble loss), we focus on the gap in error rates on a relative scale. As we show, this provides \emph{much} finer insights into ensembles improvement. To complement this, \cite{fort2019deep} study ensembling from a loss landscape perspective, evaluating how different approaches to ensembling, such as deep ensembles, Bayesian ensembles, and local methods like Gaussian subspace sampling compare in function and weight space diversity. Other recent work has studied the use of ensembling to provide uncertainty estimates for prediction \cite{lakshminarayanan2017simple}, and to improve robustness to out-of-distribution data \cite{ovadia2019can}, although the ubiquity of these findings has recently been questioned in \cite{abe2022deep}.

\vspace{-1mm}
\section{Ensemble improvement, competence, and the disagreement-error ratio}
\label{sec:theory}
\vspace{-1mm}
In this section, our goal is to characterize theoretically the rate at which ensembling improves performance. To do this, we first need to formalize a metric to quantify the benefit from ensembling. 
One natural way of measuring this improvement would be to compute the gap $\E_{h\sim \Qc}[\testerr(h)] - \testerr(h_\mv)$. 
A similar gap was the focus of \cite{abebest}, although in terms of the cross-entropy loss, rather than the classification error rate. 
However, the unnormalized gap can be misleading---in particular, it
will tend to be small whenever the average error rate itself is small, thus making it impractical to compare, e.g., across tasks of varying difficulty. 
Instead, we work with a normalized version of the average-minus-ensemble error rate gap, where the effect of the normalization is to measure this error in a relative scale.
We call this the ensemble improvement rate. 
%\newpage
\begin{definition}[Ensemble improvement rate]
\label{def:eir}
Given distributions $\rho$ over classifiers and $\Dc$ over data, provided that $\E_{h\sim \Qc}[\testerr(h)]\neq 0$, the \defn{ensemble improvement rate (EIR)} is defined as
\vspace{-1mm}
\begin{align*}
\EI = \frac{\E_{h\sim \Qc}[\testerr(h)] - \testerr(h_\mv)}{\E_{h\sim \Qc}[\testerr(h)]}.
\end{align*}
\end{definition}
In contrast to the unnormalized gap, the ensemble improvement rate can be large even for very easy tasks with a small average error rate. Recall the simple first-order bound on the majority-vote classifier: $L(h_\mv) \leq 2\E[L(h)]$. Rearranging, we deduce that $\EI \geq -1$. Unfortunately, in the absence of additional information, this first-order bound is in fact tight: one can construct ensembles for this $L(h_\mv) = 2\E[L(h)]$ (see Appendix \ref{app:pathological}). However, this bound is inconsistent with how ensembles generally behave and practice, and indeed it tells us nothing about when ensembling can improve performance. In the subsequent sections, we derive improved bounds on the $\EI$ that do. 

\subsection{Competent ensembles never hurt}
Surprisingly, to our knowledge, there is no known characterization of the majority-vote error rate that \emph{guarantees} it can be no worse than the error rate of any individual classifier, on average. Indeed, it turns out this is the result of strange behavior that can arise for particularly pathological ensembles rarely encountered in practice (see Appendix \ref{app:pathological} for a more detailed discussion of this). To eliminate these cases, we introduce a mild condition that we call \emph{competence}.  
\begin{assumption}[Competence]
\label{assumption-1}
Let $W_\rho \equiv W_\rho(X,Y) = \mathbb{E}_{h\sim \rho}[\mathbb{1}(h(X) \neq Y)]$. 
The ensemble $\rho$ is \defn{competent} if for every $0 \leq t \leq 1/2$, 
\[
\mathbb{P}(W_{\rho} \in [t,1/2)) \geq \mathbb{P}(W_{\rho} \in [1/2,1-t]).
\]
\end{assumption}

The competence assumption guarantees that the ensemble is not pathologically bad, and in particular it eliminates the scenarios under which the naive first-order bound (\ref{eq:trivial-FO}) is tight. As we will show in Section \ref{sec:evaluating-theory}, the competence condition is quite mild, and it holds broadly in practice. Our first result uses competence to improve non-trivially the naive first-order bound.
\begin{theorem}
\label{thm:Always}
Competent ensembles never hurt performance, i.e., $\EI \geq 0$.
\end{theorem}

Translated into a bound on the majority vote classifier, Theorem \ref{thm:Always} guarantees that $L(h_\mv)\leq \E[L(h)]$, improving on the naive first-order bound (\ref{eq:trivial-FO}) by a factor of two. To the best of our knowledge, the competence condition is the first of its kind, in that it guarantees what is widely observed in practice, i.e., that ensembling cannot \emph{hurt} performance. 
However, it is insufficient to answer the question of \emph{how much} ensembling can improve performance. To address this question, we turn to a "second-order" analysis involving the disagreement~rate. 

\subsection{Quantifying ensemble improvement with the disagreement-error ratio}
\label{sec:eir-der}

Our central result in this section will be to relate the $\EI$ to the ratio of the disagreement to average error rate, which we define formally below. 
%Our central result in this section will be to relate the $\EI$ to the ratio of disagreement to average error rate, which for convenience we define below.

\begin{definition}[Disagreement-error ratio]
\label{def:der}
    Given the distributions $\rho$ over classifiers and $\Dc$ over data, provided that $\E_{h\sim \Qc}[\testerr(h)]\neq 0$, the \defn{disagreement-error ratio (DER)} is defined as 
    \begin{align*}
        \DER = \frac{\E_{h,h'\sim \Qc}[\dis(h,h')]}{\E_{h\sim \Qc}[\testerr(h)]}.
    \end{align*}
\end{definition}

Our next result relates the $\EI$ to a linear function of the $\DER$.
\begin{theorem}
\label{thm:eir-der-linear}
For any competent ensemble $\rho$ of $K$-class classifiers, provided $\E_{h\sim \Qc}[\testerr(h)] \neq 0$, the ensemble improvement rate satisfies
\begin{align*}
 \DER \geq \EI \geq \frac{2(K-1)}{K}\DER - \frac{3K - 4}{K}.
\end{align*}
\end{theorem}

Note that neither Theorem \ref{thm:Always} nor Theorem \ref{thm:eir-der-linear} is uniformly stronger. In particular, if $\DER < (3K-4)/(2K-2)$ then the lower bound provided in Theorem \ref{thm:Always} will be superior to the one in Theorem \ref{thm:eir-der-linear}. 
% Thus we can summarize our theoretical characterization of the $\EI$ succinctly as follows:

% \begin{center}
% \begin{tcolorbox}[hbox]
% $
%  \DER \geq \EI \geq \max\left\{\frac{2(K-1)}{K}\DER - \frac{3K - 4}{K},0\right\}.
% $
% \end{tcolorbox}
% \end{center}

Theorem \ref{thm:eir-der-linear} predicts that the $\EI$ is fundamentally governed by a linear relationship with the $\DER$ --- a result that we will verify empirically in Section \ref{sec:evaluating-theory}. Importantly, we note that there are two distinct regimes in which the bounds in Theorem \ref{thm:eir-der-linear} provide non-trivial guarantees. 

\begin{itemize}
    \item[] 
    $\mathbf{DER}$ \textbf{small} ($\mathbf{< 1}$). 
    In this case, by the trivial bound (\ref{eq:trivial-FO}), $\EI \leq 1$, and thus the upper bound in Theorem \ref{thm:eir-der-linear} guarantees ensemble improvement cannot be too large whenever $\DER < 1$, that is, whenever disagreement is small relative to the average error rate.
    \item[] 
    $\mathbf{DER}$ \textbf{large} ($\mathbf{> 1}$).  
    In this case, the lower bound in Theorem \ref{thm:eir-der-linear} guarantees ensemble improvement whenever disagreement is sufficiently large relative to the average error rate, in particular when $\DER \geq (3K-4)/(2K-2) \geq 1$.
\end{itemize}

In our empirical evaluations, we will see that these two regimes ($\DER > 1$ and $\DER < 1$) strongly distinguish between situations in which ensemble improvement is high, and when the benefits of ensembling are significantly less pronounced. 

Moreover, Theorem \ref{thm:eir-der-linear} captures an important subtlety in the relationship between ensemble improvement and predictive diversity. In particular, while general intuition---and a significant body of prior literature, as discussed in Section \ref{sec:background}---suggests that higher disagreement leads to high ensemble improvement, this may \emph{not} be the case if the disagreement is nominally large, but small relative to the average error rate. 
% In Section \ref{sec:case-studies}, we provide a strong example of this for ensembles of fine-tuned BERT models, wherein the disagreement can be nominally large ($\approx 30\%$), but the disagreement-error ratio is small ($<1$), and ensembling is of little benefit.

\begin{remark}[Corollaries of Theorems \ref{thm:Always} and \ref{thm:eir-der-linear}]
Using some basic algebra, the upper and lower bounds presented in Theorems \ref{thm:Always} and \ref{thm:eir-der-linear} can easily be translated into upper and lower bounds on the error rate of the majority vote classifier itself. For the sake of space, we defer discussion of these Corollaries to Appendix \ref{app:mv_bounds}, although we note that the resulting bounds constitute significant improvements on existing bounds, which we verify both analytically (when possible) and empirically.
\end{remark}

% \begin{remark}[Improving the lower bound]
%     Despite significantly improving on prior bounds, the lower bound in Theorem \ref{thm:eir-der-linear} can be further improved. Specifically, Theorem \ref{thm:eir-der-linear}'s dependence on $K$ is sub-optimal when the ensemble's predictions for a given input are, with high probability, concentrated on a small number of classes (as is typically the case in practice). In Appendix \ref{app:additional-theory}, we show how the bound can be significantly improved under such conditions.
% \end{remark}

\section{Evaluating the theory}
\label{sec:evaluating-theory}

In this section, we investigate the assumptions and predictions of the theory proposed in Section \ref{sec:theory}. In particular we will show 1) that the competence assumption holds broadly in practice, across a range of architectures, ensembling methods and datasets, and 2) that the $\EI$ exhibits a close linear relationship with the $\DER$, as predicted by Theorem \ref{thm:eir-der-linear}.

Before presenting our findings, we first briefly describe the experimental settings analyzed in the remainder of the paper. Our goal is to select a sufficiently broad range of tasks and methods so as to demonstrate the generality of our conclusions. 

\subsection{Setup for empirical evaluations}
\label{sec:experimental-setup}

In Table \ref{tab:models} we provide a brief description of our experimental setup; more extensive experimental details can be found in Appendix \ref{app:emp-details}.

% \begin{table*}[ht!]
% \centering

% \caption{}
% \label{tab:datasets}
% \end{table*}

\begin{table*}[h]

\caption{Datasets and ensembles used in empirical evaluations, where $C$ denotes the number of classes, and $M$ denotes the number of classifiers.}
\centering
\scalebox{0.9}{
{\small
\setlength\tabcolsep{6pt}
\begin{tabular}{|l|l|l|l|}
\hline 
\multicolumn{3}{|c|}{\textbf{\emph{Datasets}}} \\
\hline
\multicolumn{1}{|l|}{\textbf{Dataset}} & \multicolumn{1}{|l|}{\textbf{C}} & \multicolumn{1}{|l|}{\textbf{Reference}} \\
\hline
\hline
MNIST (5K subset)           & 10                  &  \cite{mnist}            \\
CIFAR-10         & 10                   & \cite{cifar10}         \\
% CIFAR-10.1         & 10                   & \cite{recht2018cifar10.1, torralba2008tinyimages}         \\
% CIFAR-10-C         & 10                   & \cite{hendrycks2019robustness}         \\
IMDB             & 2                    & \cite{imdb-dataset}    \\
QSAR             & 2                   &  \cite{qsar-dataset}    \\
Thyroid          & 2                   & \cite{thyroid-dataset} \\
GLUE (7 tasks) & 2-3 & \cite{wang2018glue}\\
\hline
\end{tabular}
}
{\small
\setlength\tabcolsep{6pt}
\begin{tabular}{|l|l|l|l|}
\hline 
\multicolumn{4}{|c|}{\textbf{\emph{Ensembles}}} \\
\hline
 \multicolumn{1}{|l|}{\textbf{Base classifier}} & \multicolumn{1}{|l|}{\textbf{Ensembling}} & \multicolumn{1}{|l|}{\textbf{M}} & \multicolumn{1}{|l|}{\textbf{Reference}} \\
\hline
\hline
ResNet20-Swish           & Bayesian Ens.     &   100         & \cite{pmlr-v139-izmailov21a}            \\
ResNet18         & Deep Ens.           & 5       & \cite{cifar10}         \\
CNN-LSTM             & Bayesian Ens.                  & 100 & \cite{pmlr-v139-izmailov21a}    \\
BERT (fine-tune) & Deep Ens. & 25 & \cite{devlin-etal-2019-bert, multibert}\\
Random Features            & Bagging                  & 30 & N/A    \\
Decision Trees             & Random Forests          & 100         & \cite{scikit-learn, breiman2001random}  \\
\hline
\end{tabular}
}
}
\label{tab:models}
\end{table*}

\begin{figure}[h!]
    \centering
    \includegraphics[scale=0.275]{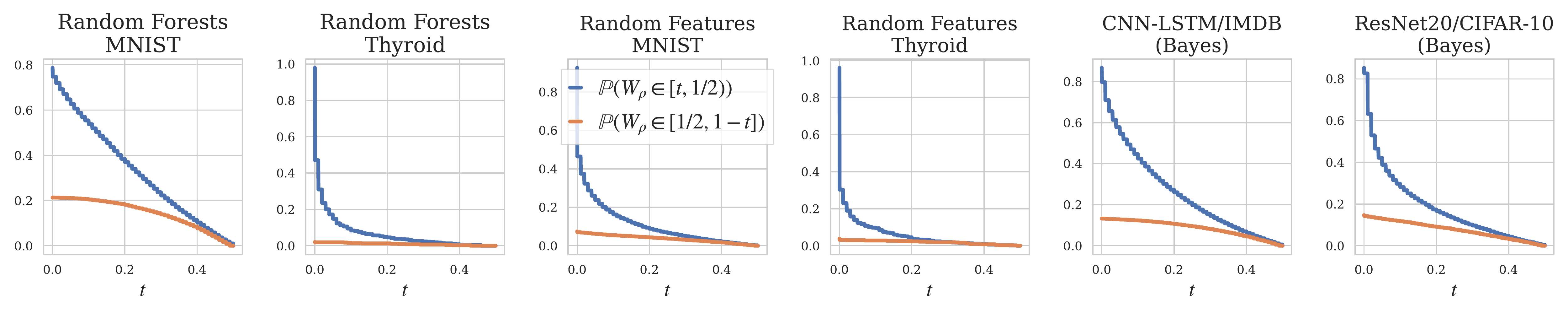}
    \caption{\textbf{Verifying the competence assumption in practice.} $W_\rho(X,Y)$ in Assumption \ref{assumption-1} is estimated using hold-out data. Across all tasks, \fcolorbox{black}{blue}{\rule{0pt}{4pt}\rule{4pt}{0pt}} $>$ \fcolorbox{black}{orange}{\rule{0pt}{4pt}\rule{4pt}{0pt}}, supporting Assumption \ref{assumption-1}.}\vspace{-10pt}
    \label{fig:competence}
\end{figure}

\subsection{Verifying competence in practice.} 
Our theoretical results in Section \ref{sec:theory} relied on the competence condition. 
One might wonder whether competent ensembles exist, and if so how ubiquitous they are. 
Here, we test that assumption. 
(That is, we test not just the predictions of our theory, but also the assumptions of our~theory.)

We have observed that the competence assumption is empirically very mild, and that in practice it applies very broadly. 
In Figure \ref{fig:competence}, we estimate both $\P(W_\rho \in [t,1/2))$ and $\P(W_\rho \in [1/2, 1-t])$ on test data, validating that competence holds for various types of ensembles across a subset of tasks. 
To do this, given a test set of examples $\{(\xb_j,y_j)\}_{j=1}^m$ and classifiers $h_1,\dots,h_N$ drawn from $\rho$, we construct the estimator 
$$
\widehat{W}_\rho^{(j)} = \frac{1}{N}\sum_{n=1}^N \mathbb{1}(h_n(\xb_j)\neq y_j),
$$
and we calculate $\P(W_\rho \in [t,1/2))$ and $\P(W_\rho \in [1/2, 1-t])$ from the empirical CDF of $\{\widehat{W}_\rho^{(j)}\}_{j=1}^m$. 
In Appendix~\ref{app:additional-empirical}, we provide additional examples of competence plots across more experimental settings (and we observe substantially the same results).

\subsection{The linear relationship between $\DER$ and $\EI$.}

\begin{wrapfigure}[20]{r}{0.42\textwidth}
    \centering
    %\vspace{-10pt}
    \raisebox{0pt}[\dimexpr\height- 1\baselineskip\relax]{
    \includegraphics[scale=0.24]{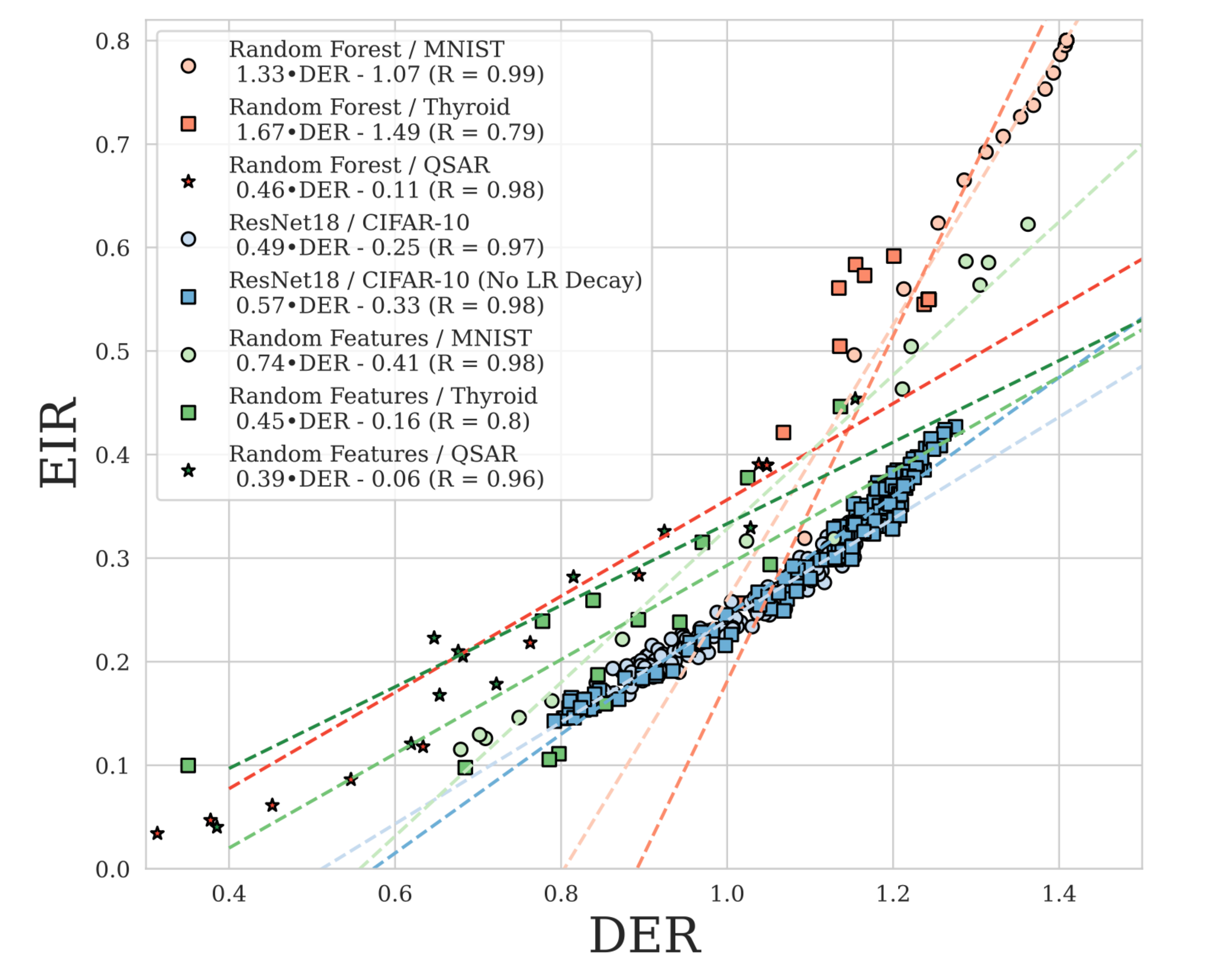}}
    
    %\includegraphics[scale=0.2]{figs/EIR_vs_DER_all_exp_marginfixed.pdf}
    
    %\vspace{-1pt}
    \caption{\textbf{$\EI$ is linearly correlated with the $\DER$.} We plot the $\EI$ against the $\DER$ across a variety of experimental settings, and we observe a close linear relationship between the $\EI$ and $\DER$, as predicted by our theoretical results. In the legend, we also report the equation for the line of best fit within each setting, as well as the Pearson correlation.}
    \label{fig:EIR_vs_DER_scatter}
\end{wrapfigure}

Theorem \ref{thm:eir-der-linear} predicts a linear relationship between the $\EI$ and the $\DER$; here we verify that this relationship holds empirically. In Figure \ref{fig:EIR_vs_DER_scatter}, we plot the $\EI$ against the $\DER$ across several experimental settings, varying capacity hyper-parameters (width for the ResNet18 models, number of random features for the random feature classifiers, and number of leaf nodes for the random forests), reporting the equation of the line of best fit, as well as the Pearson correlation between the two metrics. Across 6 of the 8 experimental settings evaluated, we find a very strong linear relationship -- with the Pearson $R \geq 0.96$. The two exceptions are found for the Thyroid classification dataset (though there is still a strong trend between the two quantities, with $R \approx 0.8$). 

Interestingly, we can compare the lines of best fit to the theoretical linear relationship predicted by Theorem \ref{thm:eir-der-linear}. In the case of the binary classification datasets (QSAR and Thyroid), the bound predicts that $\EI \approx \DER - 1$; while for the 10-class problems (MNIST and CIFAR-10), the equation is $\EI \approx 1.8\,\DER - 2.6$. While for some examples (e.g., random forests on MNIST), the equations are close to those theoretically predicted, there is a clear gap between theory and the experimentally measured relationships for other tasks. In particular, we notice a significant difference in the governing equations for the same datasets between random forests and the random feature ensembles, suggesting that the relationship is to some degree modulated by the model architecture -- something our theory cannot capture. We therefore see refining our theory to incorporate such information as an important direction for future work.

\vspace{-2mm}

\section{Ensemble improvement is low in the interpolating regime}
\label{sec:interpolation}

In this section, we will show that the $\DER$ behaves qualitatively differently for interpolating versus non-interpolating ensembles, in particular exhibiting behavior associated with phase transitions.
Such phase transitions are well-known in the statistical mechanics approach to learning~\cite{EB01_BOOK,MM17_TR,pmlr-v130-theisen21a,YaoqingTaxonomizing}, but they have been viewed as surprising from the more traditional approach to statistical learning theory~\cite{BHMM19}.
We will use this to understand when ensembling is and is not effective for deep ensembles, and to explain why tree-based methods seem to benefit so much from ensembling across all settings.
We say that a model is \emph{interpolating} if it achieves exactly zero training error, and we say that it is \emph{non-interpolating} otherwise; we call an ensemble interpolating if each of its constituent classifiers is interpolating. Note that for methods that involve resampling of the training data (e.g., bagging methods), we define the training error as the ``in-bag'' training, i.e., the error evaluated only on the points a classifier was trained on.

\begin{figure}[t]
\centering
\begin{minipage}{0.48\textwidth}
    \centering
    \includegraphics[scale=0.38]{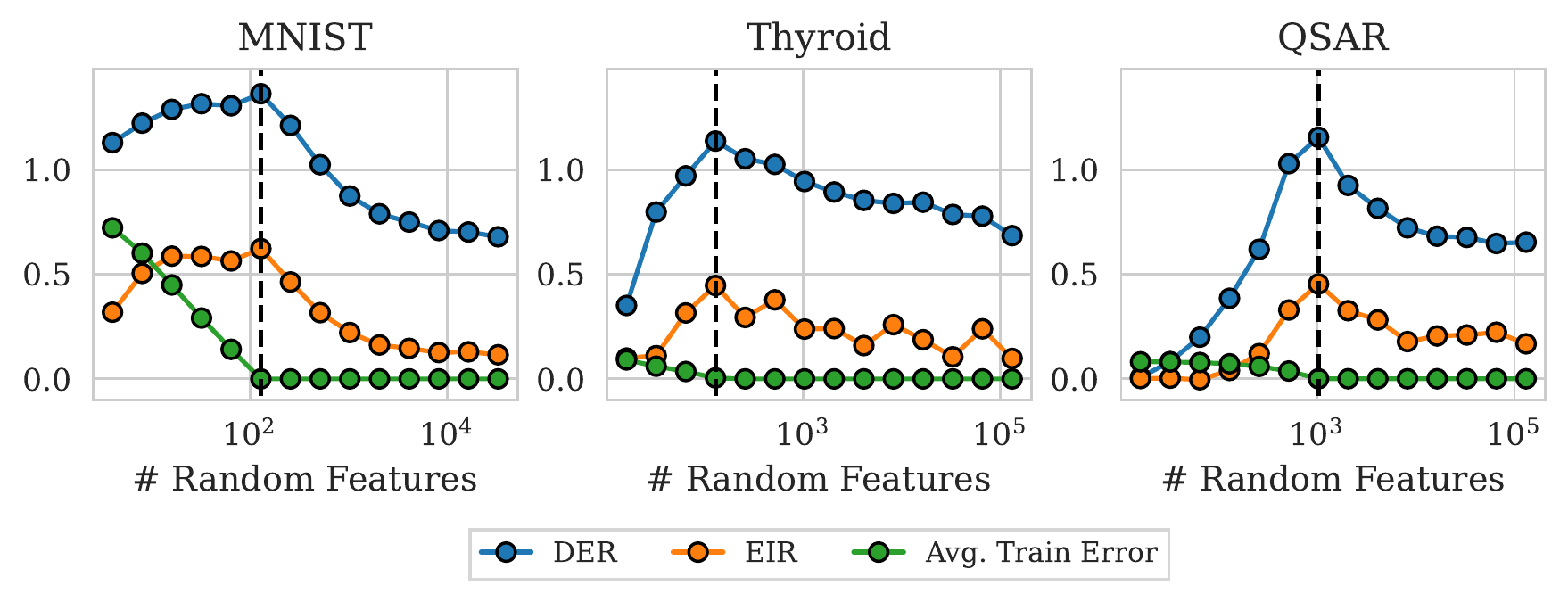}
    \captionof{figure}{\textbf{Bagged random feature classifiers.} Blacked dashed line represents the interpolation threshold. Across all tasks, $\DER$ and $\EI$ are maximized at this point, and then decrease thereafter.}
    \label{fig:brrf-interpolation}
\end{minipage}
\hfill
\begin{minipage}{0.48\textwidth}
    \centering
    \includegraphics[scale=0.38]{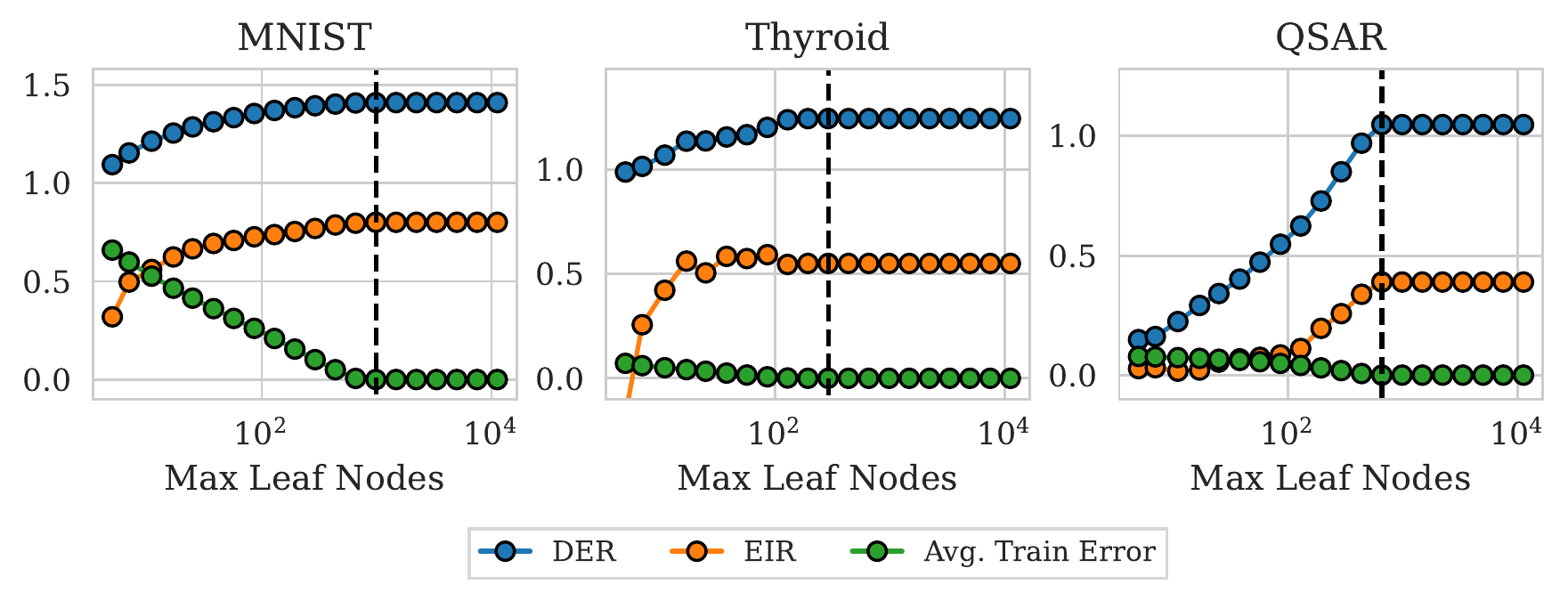}
    \captionof{figure}{\textbf{Random forest classifiers.} Blacked dashed line represents the interpolation threshold. Across all tasks, $\DER$ and $\EI$ are maximized at this point, and then remain constant thereafter.}
    \label{fig:rf-interpolation}
\end{minipage}
\vspace{-2mm}
\end{figure}
\begin{figure}[t]
    \centering
    \begin{subfigure}{0.48\textwidth}
        \includegraphics[scale=0.54]{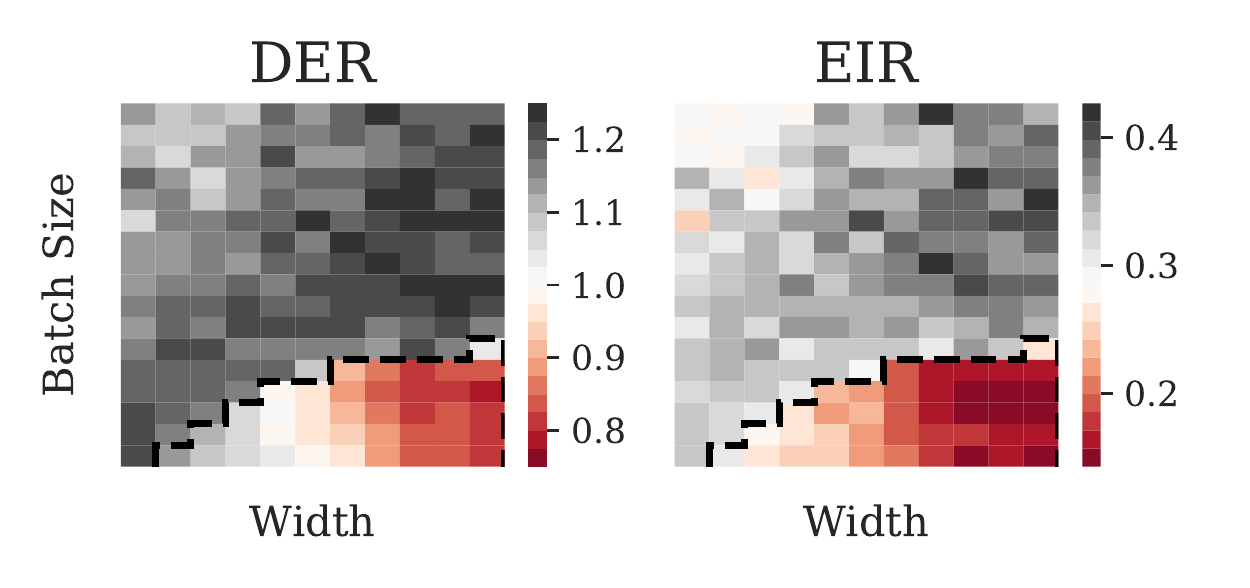} \vspace{-2mm}
        \caption{\textbf{Without LR decay.}}
    \end{subfigure}
    \hfill 
    \begin{subfigure}{0.48\textwidth}
        \includegraphics[scale=0.54]{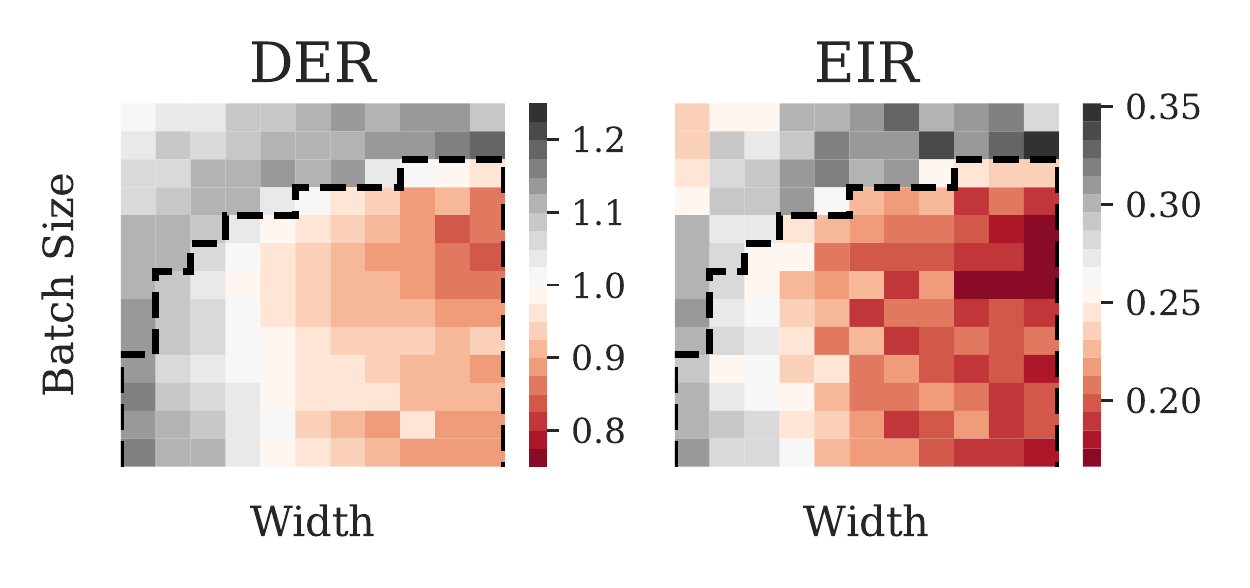} \vspace{-2mm}
        \caption{\textbf{With LR decay.}}
    \end{subfigure}
    \caption{\textbf{Large scale studies of deep ensembles on ResNet18/CIFAR-10.} We plot the $\DER$, $\EI$, average error and majority vote error rate across a range of hyper-parameters, for two training settings: one with learning rate decay, and one without. The black dashed line indicates the \emph{interpolation threshold}, i.e., the curve below which individual models achieve exactly zero training error. Observe that interpolating ensembles attain distinctly lower $\EI$ than non-interpolating ensembles, and correspondingly have low $\DER$ ($<1$), compared to non-interpolating ensembles with high $\DER$ ($>1$).
    } \vspace{-5pt}
    \label{fig:resnet-interpolation}
\end{figure}
\begin{figure}[t]
\centering
\begin{minipage}{0.42\textwidth}
    \centering
    \includegraphics[scale=0.38]{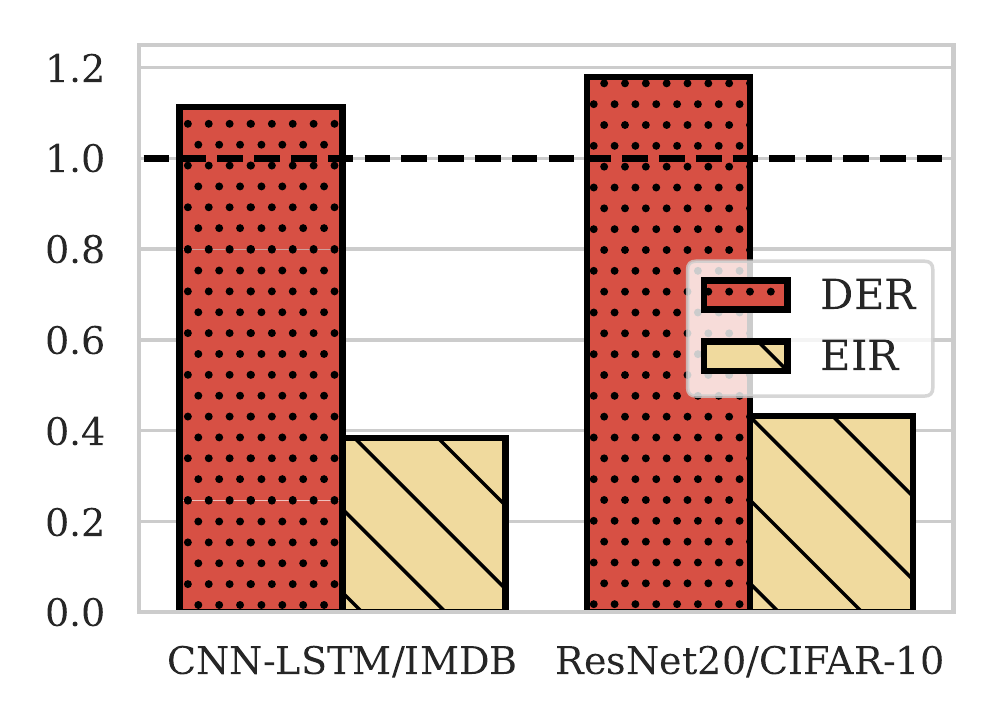}
    \vspace{-3mm}
    \captionof{figure}{\textbf{Bayesian ensembles on IMDB and CIFAR-10.} These provide examples of ensembles which do \emph{not} interpolate the training data, and which have high $\DER$ ($>1$) and correspondingly high~$\EI$.}
    \label{fig:bayes-imdb-cifar10}
\end{minipage}
\hspace{5mm}
\begin{minipage}{0.53\textwidth}
    \centering
    \includegraphics[scale=0.41]{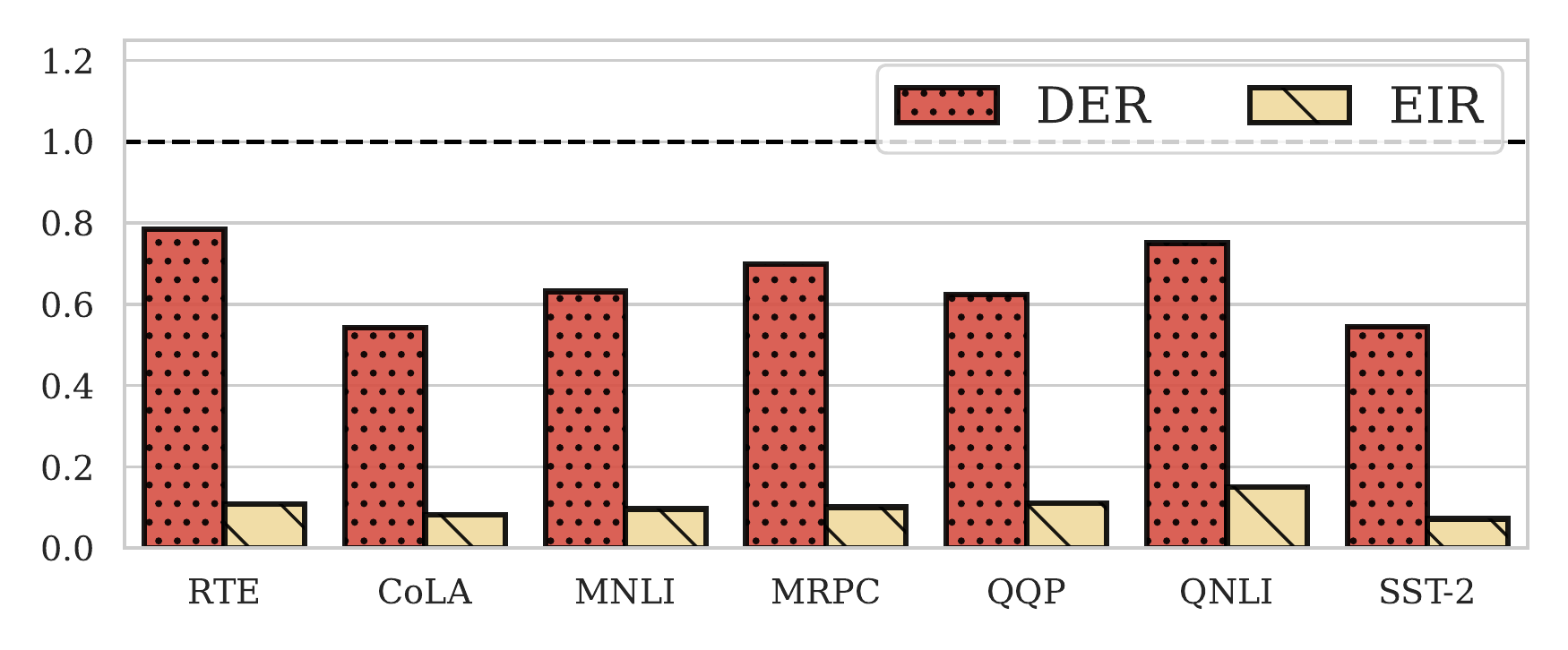} \vspace{-2mm}
    \caption{\textbf{Ensembles of fine-tuned BERT on GLUE tasks.} Here the models are large relative to the dataset size, and consequently they exhibit low $\DER$ ($<1$) and $\EI$ across all tasks.}
    \label{fig:bert-glue}
\end{minipage}

\end{figure}

\paragraph{Interpolating random feature classifiers.}
We first look at the bagged random feature ensembles on the MNIST, Thyroid, and QSAR datasets. 
In Figure \ref{fig:brrf-interpolation}, we plot the $\EI$, $\DER$ and training error for each of these ensembles (recall that for ensembles which use bagging, the training error is computed as the \emph{in-bag} training error). 
We observe the same phenomenon across these three tasks: 
\emph{as a function of model capacity, the $\EI$ and $\DER$ are both maximized at the interpolation threshold, before decreasing thereafter.
This indicates that much higher-capacity models, those with the ability to easily interpolate the training data, benefit significantly less from ensembling.
In particular, observe that for sufficiently high-capacity ensembles, the $\DER$ become less than 1, entering the regime in which our theory guarantees low ensemble improvement.}

\paragraph{Interpolating deep ensembles.}  
We next consider the $\DER$, $\EI$, average test error, and majority vote test error rates for large-scale empirical evaluations on ResNet18/CIFAR-10 models in batch size/width space, both with and without learning rate decay. 
See Figure \ref{fig:resnet-interpolation} for the results.
Note that the use of learning rate decay facilitates easier interpolation of the training data during training, hence broadening the range of hyper-parameters for which interpolation occurs.
The figures are colored so that ensembles in the regime $\DER < 1$ are in red, while ensembles with $\DER > 1$ are in grey. 
The black dashed line indicates the interpolation threshold, i.e., the curve in hyper-parameter space below which ensembles achieve zero training error (meaning every classifier in the ensemble has zero training error). 
\emph{Observe in particular that, across all settings, \emph{all} models in the regime $\DER < 1$ are interpolating models; and, more generally, that interpolating models tend to exhibit \emph{much} smaller $\DER$ than non-interpolating models.} 
Observe moreover that there can be a sharp transition between these two regimes, wherein the $\DER$ is large just at the interpolation threshold, and then it quickly decreases beyond that threshold. 
Correspondingly, ensemble improvement is \emph{much} less pronounced for interpolating versus more traditional non-interpolating ensembles. 
This is consistent with results previously observed in the literature, e.g., \cite{geiger2020scaling}. 
We remark also that the behavior observed in Figure \ref{fig:resnet-interpolation} exhibits the same phases identified in \cite{YaoqingTaxonomizing} (an example of phase transitions in learning more generally~\cite{EB01_BOOK,MM17_TR}), although the $\DER$ itself was not considered in that previous~study.

\paragraph{Bayesian neural networks.}
Next, we show that Bayesian neural networks benefit significantly from ensembling. In Figure \ref{fig:bayes-imdb-cifar10}, we plot the $\DER$ and $\EI$ for Bayesian ensembles on the CIFAR-10 and IMDB tasks, using the ResNet20 and CNN-LSTM architectures, respectively. The samples we present here are provided by \cite{izmailov2021dangers}, who use Hamiltonian Monte Carlo to sample accurately from the posterior distribution over models. For both tasks, we observe that both ensembles exhibit $\DER > 1$ and high $\EI$. In light of our findings regarding ensemble improvement and interpolation, we note that the Bayesian ensembles by design do \emph{not} interpolate the training data (when drawn at non-zero temperature), as the samples are drawn from a distribution not concentrated only on the modes of the training loss. While we do not perform additional experiments with Bayesian neural networks in the present work, evaluating the $\DER$/$\EI$ as a function of posterior temperature for these models is an interesting direction for future work. We hypothesize that the qualitative effective of decreasing the sampling temperature will be similar to that of increasing the batch size in the plots in {Figure~\ref{fig:resnet-interpolation}.}

\paragraph{Fine-tuned BERT ensembles.}
Here, we present results for ensembles of BERT models fine-tuned on the GLUE classification tasks. These provide examples of a very large model trained on small datasets, on which interpolation is easily possible. For these experiments, we use 25 BERT models pre-trained from independent initializations provided in \cite{multibert}. Each of these 25 models is then fine-tuned on the 7 classification tasks in the GLUE \cite{wang2018glue} benchmark set and evaluated on relatively small test sets, ranging in size from ~250 (RTE) to ~40,000 (QQP) samples. In Figure \ref{fig:bert-glue}, we plot the $\EI$ and $\DER$ across these benchmark tasks and observe that, as predicted, the ensemble improvement rate is low and $\DER$ is uniformly low ($<1$).

\paragraph{The unique case of random forests.}
Random forests are one of the most widely-used ensembling methods in practice. 
%PREVIOUS VERSION% Here, we show that their effectiveness is much more universal than for highly-parameterized models like the random feature classifiers and the deep ensembles. 
Here, we show that the effectiveness of ensembling is much greater for random forest models than for highly-parameterized models like the random feature classifiers and the deep ensembles. 
Note that for random forests, interpolation of the training data \emph{is} possible, in particular whenever the number of terminal leaf nodes is sufficiently large (where here we again compute the average training error using on the in-bag training examples for each tree), 
but it is not possible to go ``into'' the interpolating regime. 
In Figure \ref{fig:rf-interpolation}, we plot the $\DER$, $\EI$, and training error as a function of the max number of leaf nodes (a measure of model complexity). 
Before the interpolation threshold, both the $\EI$ and $\DER$ increase as a function of model capacity, in line with what is observed for the random feature and deep ensembles. 
However, we observe distinct behavior at the interpolation threshold: both $\EI$ and $\DER$ become constant past this threshold. 
This is fundamental to tree-based methods, due to the method by which they are fit, e.g., using a standard procedure like CART \cite{trees-book-breiman}. 
As soon as a tree achieves zero training error, any impurity method used to split the nodes further is saturated at zero, and therefore the models cannot continue to grow. 
This indicates that trees are particularly well-suited to ensembling across all hyper-parameter values, in contrast to other parameterized types of classifiers. 
\section{Discussion and conclusion}

To help answer the question of when ensembling is effective, we introduce the ensemble improvement rate ($\EI$), which we then study both theoretically and empirically. Theoretically, we provide a comprehensive characterization of the $\EI$ in terms of the disagreement-error ratio $\DER$. 
The results are based on a new, mild condition called \emph{competence}, which we introduce to rule out pathological cases that have hampered previous theoretical results. 
Using a simple first-order analysis, we show that the competence condition is sufficient to guarantee that ensembling cannot hurt performance---something widely observed in practice, but surprisingly unexplained by existing theory. 
Using a second-order analysis, we are able to theoretically characterize the $\EI$, by upper and lower bounding it in terms of a linear function of the $\DER$. On the empirical side, we first verify the assumptions of our theory (namely that the competence assumption holds broadly in practice), and we show that our bounds are indeed descriptive of ensemble improvement in practice. We then demonstrate that improvement decreases precipitously for interpolating ensembles, relative to non-interpolating ones, providing a very practical guideline for when to use ensembling. 
% Finally, we perform case studies evaluating different practical ensembling scenarios. 
% For example, we find that the behavior of the $\DER$/$\EI$ are largely robust to out-of-distribution perturbations, and that ensembling is largely ineffective when fine-tuning, even from independently pre-trained models. 

% By combining our theoretical results and our empirical results, one would like to conclude that ensembling is not really effective in the ``modern overparameterized'' regime.
% Our results do suggest that something like this is the correct intuition, with a phase transition in effectiveness between the two regimes, but our results also indicate important subtleties.
% For instance, if one prefers to use Bayesian neural networks, then the ``form factor'' of the Bayesian approach leads to ``noisier'' models, and in this case ensembling can still be effective.
% Regardless, though, our results do highlight how intuitions from the classical regime, in which the number of data points far exceeds the number of parameters, need to be revisited---both for ensemble methods as well as for related tasks such as uncertainty quantification.
% For instance, if one interpolates the data, then one is implicitly assuming a noise model for the data, namely that there is no noise.  
% While no one would suggest that modern data have no noise, that is interesting as a modeling assumption, and the implications of it for Bayesian-style modeling remain to be explored.

Our work leaves many directions to explore, of which we name a few promising ones. First, while our theory represents a significant improvement on previous results, there are still directions to extend our analysis. For example, Figure \ref{fig:EIR_vs_DER_scatter} suggests that the relationship between $\EI$ and $\DER$ can be even more finely characterized. Is it possible to refine our analysis further to incorporate information about the data and/or model architecture? Second, can we formalize the connection between ensemble effectiveness and the interpolation point, and relate it to similar ideas in the literature? 

% \section{Societal impact}
% Our research centers around developing a theoretical understanding of ensembling. Although it can be applied to ensembling techniques with potential adverse applications, we do not see any immediate negative societal impacts stemming from the theory itself.
% \cite{*}

\paragraph{\textbf{Acknowledgments.}}
We would like to acknowledge the DOE, IARPA, NSF, and ONR as well as a J. P. Morgan Chase Faculty Research Award for providing partial support of this work.

\bibliographystyle{alpha}
\bibliography{references.bib}

\appendix

%\onecolumn

\section{Proofs of our main results}
\label{app:proofs}

% \begin{proof}[Proof of Proposition \ref{thm:so-lb}]
% Note that
% \begin{align*}
%     2\E[\testerr(\hb)] &= \E[\|\hb(X)-Y\|^2] \\
%     &= \E\|\hb(X)-\hb_\mv(X)\|^2 + 2\E(\hb(X)-\hb_\mv(X))^\top(\hb_\mv(X)-Y) + \E\|\hb_\mv(X)-Y\|^2 \\
%     &= \E\|\hb(X)-\hb_\mv(X)\|^2 + 2\E(\hb(X)-\hb_\mv(X))^\top(\hb_\mv(X)-Y) + 2\E[\mathbb{1}(\hb_\mv(X)\neq Y)] .
% \end{align*}
% Now we have
% $$
% \frac{1}{2}\E_{\hb}\|\hb(X) - \hb_\mv(X)\|^2 = 1-\bar{\hb}(X)^\top\hb_\mv(X) = 1-\|\bar{\hb}(X)\|_{\infty}.
% $$
% Moreover, the second term
% By H\"older's inequality, we have
% $$
% \|\bar{\hb}(X)\|^2 = |\bar{\hb}(X)^\top \bar{\hb}(X)| \leq \|\bar{\hb}(X)\|_1\|\bar{\hb}(X)\|_\infty = \|\bar{\hb}(X)\|_\infty.
% $$
% Since $\E[\dis(\hb,\hb')] = 1-\E\|\bar{\hb}(X)\|^2$, we get
% $$
% \frac{1}{2}\E\|\hb(X) - \hb_\mv(X)\|^2  \leq \E[1-\|\bar{\hb}(X)\|^2] = \E[\dis(\hb,\hb')].
% $$
% Next we claim $\E(\hb(X)-\hb_\mv(X))^\top(\hb_\mv(X)-Y)\leq 0$. 
% There are two cases.
% The first is on the set where $\hb_\mv(X) = Y$, where this is $=0$. 
% The second is on the set where $\hb_\mv(X) \neq Y$, in which case the expression is $ = \|\bar{\hb}(X)\|_{\infty} - \bar{\hb}(X)^\top Y - 1 \leq 0$.
% Thus,
% $$
% \E[\testerr(\hb)] \leq \frac{1}{2}\E\|\hb(X)-\hb_\mv(X)\|^2 + \testerr(\hb_\mv) \leq \E[\dis(\hb,\hb')] + \testerr(\hb_\mv).
% $$
% \end{proof}

In this section, we provide proofs for our main results.
Throughout the section, we denote $\P_{h, h' \sim \rho^2}$, $\P_{h \sim \rho}$, $\E_{h \sim \rho}$ by $\P_{h,h'}$, $\P_h$, $\E_h$, respectively. 
We also denote $W_\rho(X, Y)$ simply by $W_\rho$.
We typically omit explicit dependence on the data distribution $\Dc$ when it is apparent from context. 
%We will also denote $\bar{h}_j(\xb) = \E_\rho[\mathbb{1}(h(\xb) = j)]$.

%%%%%%%%%%%%%%%%%%%%%%%%%%%%%%%%%%%%%%%
% MV always perform better Proof
%%%%%%%%%%%%%%%%%%%%%%%%%%%%%%%%%%%%%%%

\subsection{Proof of Theorem \ref{thm:Always}}
\label{proof:thm:Always}

We first state and prove two lemmas that will be used in the proof of Theorem \ref{thm:Always}.
%Our first lemma states the relationship between majority-vote error $\testerr_\Dc[h_\mv]$ and the average error rate for each data point, $W_\rho(X,Y)$. 
Our first lemma states that majority-vote error $\testerr_\Dc[h_\mv]$ is upper bounded by probability of $W_\rho(X,Y)$ being large.
%%\addressed{Say what that relationship is, in words.}

\begin{lemma}
\label{lemma:mv} There is the inequality
$\testerr_\Dc[h_\mv] \leq \P_\Dc(W_\rho(X,Y) \geq 1/2),$
where $\testerr_\Dc(h) = \E_{\Dc}[\mathbb{1}(h(X)\neq Y)]$, $h_{\mv}(\xb) = \argmax_{j}\; \E_{h}[\mathbb{1}(h(\xb)=j)]$ and $W_\rho(X,Y) = \mathbb{E}_{h}[\mathbb{1}(h(X) \neq Y)]$.
\end{lemma}
\begin{proof}
For given data point $x$, $W_\rho < 1/2$ implies that we are predicting the true label correctly more than half of the time. Thus, the majority vote classifier will correctly predict the label on the data point.  
\end{proof}

Our next lemma states a property of competent classifiers which plays a crucial role in the main proof. 

\begin{lemma}
\label{lemma:stoch_dominant}
Under Assumption \ref{assumption-1} (competence), for any increasing function $h$ satisfying $h(0) = 0$,
\begin{align*}
    \E_\Dc[h(W_\rho) \mathbb{1}_{W_\rho < 1/2}] \;\geq\; \E_\Dc[h(\bar{W_\rho}) \mathbb{1}_{\bar{W_\rho} \leq 1/2}],
\end{align*}
where $\bar{W_\rho} = 1 - {W_\rho}$.
\end{lemma}
\begin{proof}
For every $x \in [0,1]$,
\begin{align*}
\P_\Dc({W_\rho} \mathbb{1}_{{W_\rho} < 1/2} \geq x) 
& = \P_\Dc({W_\rho}\in [x, 1/2)) \, \mathbb{1}_{x \leq 1/2},  \\
\P_\Dc(\bar{W_\rho} \mathbb{1}_{\bar{W_\rho} \leq 1/2} \geq x) 
& = \P_\Dc(\bar{W_\rho} \in [x, 1/2]) \, \mathbb{1}_{x \leq 1/2}
= \P_\Dc(W_\rho \in [1/2, 1-x]) \, \mathbb{1}_{x \leq 1/2}.
\end{align*}
From Assumption \ref{assumption-1}, this implies that $\P_\Dc({W_\rho} \mathbb{1}_{{W_\rho} < 1/2} \geq x) \geq \P_\Dc(\bar{W_\rho} \mathbb{1}_{\bar{W_\rho} \leq 1/2} \geq x)$ for all $x \in [0,1]$. Therefore, for any increasing function $h$ satisfying $h(0) = 0$, since $h(x\,\mathbb{1}_{x\leq c}) = h(x)\mathbb{1}_{x\leq c}$, 
\begin{align*}
\P_\Dc(h({W_\rho}) \mathbb{1}_{{W_\rho} < 1/2} \geq x) \geq \P_\Dc(h(\bar{W_\rho}) \mathbb{1}_{\bar{W_\rho} \leq 1/2} \geq x).
\end{align*}
As $W_\rho$ is non-negative, the equality $\E X = \int_0^\infty \P(X \geq x) \mathrm{d}x$ concludes the proof.
\end{proof}

With these two lemmas, we now provide the proof of Theorem~\ref{thm:Always}.

\begin{proof}[Proof of Theorem \ref{thm:Always}]
From Lemma \ref{lemma:mv} and the relation $\E_h[\testerr_\Dc(h)] = \E_\Dc[W_\rho]$ (Fubini's theorem), it suffices to show that $\P_\Dc(W_\rho \geq 1/2) \leq \E_\Dc[W_\rho]$. 
To do so, observe
\begin{align*}
    \E_\Dc[(W_\rho-1) \mathbb{1}_{W_\rho \geq 1/2}] + \E_\Dc[\bar{W_\rho} \mathbb{1}_{\bar{W_\rho} \leq 1/2}]
    = \E_\Dc[(W_\rho-1) \mathbb{1}_{W_\rho \geq 1/2}] + \E_\Dc[(1-W_\rho) \mathbb{1}_{W_\rho \geq 1/2}] = 0.
\end{align*}
Applying Lemma \ref{lemma:stoch_dominant} with $h(x) = x$,
\begin{align*}
    \E_\Dc[W_\rho] - \P_\Dc(W_\rho \geq 1/2) 
    & \geq \E_\Dc[(W_\rho-1)\mathbb{1}_{W_\rho \geq 1/2}] + \E_\Dc[W_\rho\mathbb{1}_{W_\rho < 1/2}] \\
    & \geq \E_\Dc[(W_\rho-1)\mathbb{1}_{W_\rho \geq 1/2}] + \E_\Dc[\bar{W_\rho} \mathbb{1}_{\bar{W_\rho} \leq 1/2}]
    = 0.
\end{align*}
which proves 
\begin{equation}
\begin{aligned} \label{eq:app_pf_always}
    \testerr_\Dc[h_\mv] \underset{\text{Lemma} \ref{lemma:mv}}{\leq} \P_\Dc(W_\rho \geq 1/2) {\leq} \E_\Dc[W_\rho] = \E_h[\testerr_\Dc(h)].   
\end{aligned}
\end{equation}
This implies $\EI \geq 0$.
\end{proof}

%%%%%%%%%%%%%%%%%%%%%%%%%%%%%%%%%%%%%%%
% L_MV Upper bound Proof
%%%%%%%%%%%%%%%%%%%%%%%%%%%%%%%%%%%%%%%

\subsection{Proof of Theorem \ref{thm:eir-der-linear}}
\label{proof:thm:eir-der-linear}

\subsubsection{Lower bound of $\EI$} \label{app:pf_lb}
\vspace{0.1cm}

To prove the lower bound, we first define the tandem loss, as used in \cite{second-order-mv-bounds2020}.

\begin{definition}[Tandem loss]
Define the tandem loss to be $\testerr(h,h') = \E_{\Dc}[\mathbb{1}(h(X)\neq Y)\mathbb{1}(h'(X)\neq Y)]$.
\end{definition}

We also rely on the following lemma, which appears as Lemma 2 in \cite{second-order-mv-bounds2020}. It provides the connection between the average error rate for each data point, $W_\rho$ and the tandem loss, $\testerr(h,h')$.
\begin{lemma} 
\label{lemma:tandem} 
The equality $\E_{\Dc}[{W_\rho}^2] = \E_{h,h'}[\testerr(h,h')]$ holds.
\end{lemma}

We first state and  prove the following lemma, which provides an upper bound on the tandem loss.

\begin{lemma}\label{lemma:kclass_tandem_ub}
For the $K$-class problem,
\begin{align*}
    \E_{h,h'}[\testerr(h,h')] \leq \frac{2(K-1)}{K}\left(\E_h[\testerr(h)] - \frac{1}{2}\E_{h,h'}[\dis(h,h')]\right).
\end{align*}
\end{lemma}
\begin{proof}
We denote $\P_h(h(X) \neq Y)$ by $\bar{h}_Y(X)$. 
Note that $\E_\Dc(1-\bar{h}_Y(X)) = \E_h[\testerr(h)]$ and 
\begin{align*}
\E_{h,h'}[\testerr(h,h')] 
&= \E_\Dc[\P_h(h(X)\neq Y)\P_{h'}(h'(X)\neq Y)] \\
&= \E_\Dc[(1-\bar{h}_Y(X))^2].    
\end{align*}

Then we get
\begin{align*}
\E_{h,h'}[\testerr(h,h')] &= \E_\Dc[(1-\bar{h}_Y(X))^2] \\
&= 1 - \E_\Dc[\bar{h}_Y(X)] - \E_\Dc[\bar{h}_Y(X)(1-\bar{h}_Y(X))] \\
&= \E_h[\testerr(h)] - \E_\Dc[\bar{h}_Y(X)(1-\bar{h}_Y(X))].
\end{align*}

Now we will derive a lower bound of the second term. 
Since 
\[\E_{h,h'}[\mathbb{1}(h(X)\neq h'(X))] = \sum_j \bar{h}_j(X)(1-\bar{h}_j(X)) ,\]
it follows that
\begin{align*}
\bar{h}_Y(X)(1-\bar{h}_Y(X)) = \E_{h,h'}[\mathbb{1}(h(X)\neq h'(X))] - \sum_{j\neq Y}\bar{h}_j(X)(1-\bar{h}_j(X)) .
\end{align*}

By maximizing $\sum_{j\neq Y}\bar{h}_j(X)(1-\bar{h}_j(X))$ subject to $\sum_{j\neq Y} \bar{h}_j(X) = 1-\bar{h}_Y(X)$, we get $\bar{h}_j(X) = \frac{1-\bar{h}_Y(X)}{K-1}$, which yields the upper bound 
\[
\sum_{j\neq Y}\bar{h}_j(X)(1-\bar{h}_j(X)) \leq \frac{K-2}{K-1}(1-\bar{h}_Y(X)) + \frac{1}{K-1}\bar{h}_Y(X)(1-\bar{h}_Y(X)) .
\]

It follows that 
\[
\E_\Dc[\bar{h}_Y(X)(1-\bar{h}_Y(X))] \geq \frac{K-1}{K}\E_{h,h'}[\dis(h,h')] - \frac{K-2}{K}\E_h[\testerr(h)] ,
\]
and thus that 
\begin{align*}
\E_{h,h'}[\testerr(h,h')] &= \E_h[\testerr(h)] - \E_\Dc[\bar{h}_Y(X)(1-\bar{h}_Y(X))]\\
&\leq \E_h[\testerr(h)] - \left(\frac{K-1}{K}\E_{h,h'}[\dis(h,h')] - \frac{K-2}{K}\E_h[\testerr(h)]\right)\\
&= \frac{2(K-1)}{K}\left(\E_h[\testerr(h)] - \frac{1}{2}\E_{h,h'}[\dis(h,h')]\right).
% &\leq \frac{2K-3}{K-1}\E[\testerr(h)] - \frac{K-2}{K-1}\E[\dis(h,h')]
\end{align*}
\end{proof}

We now provide the proof for the lower bound of $\EI$ in Theorem~\ref{thm:eir-der-linear}.

\begin{proof}[Proof]
\begin{comment}
\begin{proposition}[Generalized second-order upper bound]
\label{thm:so-ub}
Suppose that Assumption \ref{assumption-1} holds. Then for any distribution $\Qc$ over classifiers, in the $K$-class classification setting, we have
\begin{align*}
    \testerr_{\Dc}(h_\mv) \leq \min\left\{\mathbb{E}_{h\sim\Qc}[\testerr_{\Dc}(h)], \frac{4(K-1)}{K}\left(\E_{h\sim \Qc}[\testerr_{\Dc}(h)] - \frac{1}{2}\E_{h,h'\sim \Qc}[\dis_\Dc(h,h')]\right)\right\}.
\end{align*}
\end{proposition}
\end{comment}
%Lemma \ref{lemma:tandem} states that $\E_\Dc[{W_\rho}^2] = \E_{h,h'}[\testerr(h,h')]$ and Lemma \ref{lemma:kclass_tandem_ub} provides an upper bound of $\E_{h,h'}[\testerr(h,h')]$. Since $\testerr(h_\mv) \leq \mathbb{P}_\Dc(W_\rho \geq 1/2)$ from Lemma \ref{lemma:mv}, it suffices to show 
We first claim that $\mathbb{P}_\Dc(W_\rho \geq 1/2) \leq 2\,\E_\Dc[{W_\rho}^2]$. % to get the upper bound of $\testerr(h_\mv)$. 
Then, we have 
\begin{align*}
\E_\Dc[(2{W_\rho}^2 - 1) \mathbb{1}_{{W_\rho} \geq 1/2}] 
& = \E_\Dc[(2(1-\bar{W_\rho})^2 - 1) \mathbb{1}_{\bar{W_\rho} \leq 1/2}] \\
& = \E_\Dc[(1 - 4\bar{W_\rho} + 2\bar{W_\rho}^2) \mathbb{1}_{\bar{W_\rho} \leq 1/2}] ,
\end{align*}
where $\bar{W_\rho} = 1 - {W_\rho}$. Therefore,
\begin{align*}
    \E_\Dc[(2{W_\rho}^2 - 1) \mathbb{1}_{{W_\rho} \geq 1/2}] + \E_\Dc[2\bar{W_\rho}^2 \mathbb{1}_{\bar{W_\rho} \leq 1/2}] 
    & = \E_\Dc[(1 - 4\bar{W_\rho} + 4\bar{W_\rho}^2) \mathbb{1}_{\bar{W_\rho} \leq 1/2}] \geq 0 .
\end{align*}
Now we apply Lemma \ref{lemma:stoch_dominant} with $h(x) = 2x^2$, to obtain
%Meanwhile, from Assumption \ref{assumption-1}, it is implied that $\E[h({W_\rho}) \mathbb{1}_{{W_\rho} < 1/2}] \geq \E[h(\bar{W_\rho}) \mathbb{1}_{\bar{W_\rho} \leq 1/2}]$ for any increasing function $h$ satisfying $h(0) = 0$. Using $h(x) = 2x^2$, we see that 
\begin{equation}
\begin{aligned} \label{eq:w_rho}
\E_\Dc[2{W_\rho}^2] - \P_\Dc({W_\rho} \geq 1/2) 
& \geq \E_\Dc[(2{W_\rho}^2 - 1) \mathbb{1}_{{W_\rho} \geq 1/2}] + \E_\Dc[2{W_\rho}^2 \mathbb{1}_{{W_\rho} < 1/2}] \\
& \geq \E_\Dc[(2{W_\rho}^2 - 1) \mathbb{1}_{{W_\rho} \geq 1/2}] + \E_\Dc[2\bar{W_\rho}^2 \mathbb{1}_{\bar{W_\rho} \leq 1/2}] \;\geq 0,
\end{aligned}
\end{equation}
which proves the claim, $\mathbb{P}_\Dc(W_\rho \geq 1/2) \leq 2\,\E_\Dc[{W_\rho}^2]$.

Now we put the claim together with Lemmas \ref{lemma:mv}, \ref{lemma:tandem}, and \ref{lemma:kclass_tandem_ub} to conclude the proof. 
\begin{equation}
\begin{aligned} \label{eq:app_pf_lb}
    \testerr(h_\mv) & \underset{\text{Lemma \ref{lemma:mv}}}{\leq} \mathbb{P}_\Dc(W_\rho \geq 1/2) \leq 2\,\E_\Dc[{W_\rho}^2] \underset{\text{Lemma \ref{lemma:tandem}}}{=} 2\,\E_{h,h'}[\testerr(h,h')] \\
    & \underset{\text{Lemma \ref{lemma:kclass_tandem_ub}}}{\leq} 
    \frac{4(K-1)}{K}\left(\E_h[\testerr(h)] - \frac{1}{2}\E_{h,h'}[\dis(h,h')]\right)
\end{aligned}
\end{equation}
Rearranging the terms, we obtain
\begin{align}\label{eq:rearrange}
    \mathbb{E}_{h}[\testerr_{\Dc}(h)] - \testerr(h_\mv)
    \geq \frac{2(K-1)}{K}\E_{h,h'}[\dis(h,h')] - \frac{3K-4}{K} \E_{h}[\testerr(h)].
\end{align}    
Dividing the both terms by $\E_{h}[\testerr(h)]$ gives the lower bound $\frac{2(K-1)}{K}\DER - \frac{3K-4}{K}$.
\end{proof}
%Putting this together with Lemmas \ref{lemma:mv}, \ref{lemma:tandem}, and \ref{lemma:kclass_tandem_ub} concludes the proof.

%%%%%%%%%%%%%%%%%%%%%%%%%%%%%%%%%%%%%%%
% L_MV Lower bound Proof
%%%%%%%%%%%%%%%%%%%%%%%%%%%%%%%%%%%%%%%

\subsubsection{Upper bound of $\EI$} \label{app:pf_ub} \vspace{0.1cm}

\begin{comment}
Next, we present a new lower bound which also related the majority-vote error rate to the average error rate and disagreement rate.
\michael{Presumably this is new to us?}
\begin{proposition}[Second-order lower bound]
\label{thm:so-lb}
    For any distribution $\Qc$ over classifiers,
    \begin{align*}
        \testerr_{\Dc}(h_\mv) &\geq \E_{h\sim \Qc}[\testerr_{\Dc}(h)] - (1 - \E_{\Dc}[\max_{k}\bar{h}_k(X)]) \\
        &\geq \E_{h\sim \Qc}[\testerr_{\Dc}(h)] - \E_{h,h'\sim \Qc}[\dis(h,h')].
    \end{align*}
\end{proposition}
\end{comment}

% Fortunately, in accordance with what is almost always observed in practice, these situations are generally pathological. 
We denote $\P_h(h(X) \neq Y)$ by $\bar{h}_Y(X)$. We have 
\begin{align*}
\E_h[\testerr(h)] - \testerr(h_\mv) &= \E_{h, \Dc}[\mathbb{1}(h(X)\neq Y) - \mathbb{1}(h_\mv(X)\neq Y)] .
\end{align*}
Now 
\begin{align*}
\mathbb{1}(h(X)\neq Y) - \mathbb{1}(h_\mv(X)\neq Y) &= \mathbb{1}(h_\mv(X) = Y)-\mathbb{1}(h(X)= Y)\\
&= \mathbb{1}(h(X)\neq h_\mv(X))\left(\mathbb{1}(h_\mv(X)= Y) - \mathbb{1}(h(X)=Y)\right)\\
&\leq \mathbb{1}(h(X)\neq h_\mv(X)) .
\end{align*}
Now notice $\E_{h, \Dc}[\mathbb{1}(h(X)\neq h_\mv(X))] = 1-\E_\Dc[\max_k \bar{h}_k(X)]$. 
Moreover, by H\"older's inequality,
\begin{align*}
\|\bar{\hb}(X)\|_2^2 \leq \max_k \bar{h}_k(X) ,
\end{align*}
and so
\begin{equation}
\begin{aligned} \label{eq:app_pf_ub}
\E_h[\testerr(h)] - \testerr(h_\mv) &\leq 1-\E_\Dc[\max_k \bar{h}_k(X)]  \\
&\leq 1- \E_\Dc[\|\bar{\hb}(X)\|^2] = \E_{h,h'}[\dis(h,h')].
\end{aligned}
\end{equation}

Dividing the both terms by $\E_{h}[\testerr(h)]$ gives the upper bound $\DER$.

%%%%%%%%%%%%%%%%%%%%%%%%%%%%%%%%%%%%%%%
% Corollary 
%%%%%%%%%%%%%%%%%%%%%%%%%%%%%%%%%%%%%%%

\subsection{Upper and lower bounds on the error rate of the majority vote classifier}
\label{app:mv_bounds}
\vspace{0.1cm}
We now present upper and lower bound on the majority vote classifier that follow from the bounds in Theorem \ref{thm:Always} and \ref{thm:eir-der-linear}, and compare them with existing bounds in the literature.
\begin{theorem}\label{thm:mv_error}
    For any competent ensemble $\rho$ of $K$-class classifiers, the majority vote error rate satisfies
\begin{align*}
    \testerr(h_\mv) & \leq \min\left\{\frac{4(K-1)}{K}\left(\E_{h\sim \Qc}[\testerr(h)] - \frac{1}{2}\E_{h,h'\sim \Qc}[\dis(h,h')]\right)\; , \E_{h\sim\rho}[L(h)] \right\} \\
    \testerr(h_\mv) & \geq \E_{h\sim\rho}[L(h)] - \E_{h,h'\sim \rho}[D(h,h')]. \nonumber
\end{align*}
\end{theorem}
\begin{proof}
    The upper bound follows from inequality \eqref{eq:app_pf_always} and \eqref{eq:app_pf_lb}. The lower bound follows from inequality \eqref{eq:app_pf_ub}.
\end{proof}
We have already discussed that the bound $L(h_\mv) \leq \E[L(h)]$ represents an improvement by a factor of 2 over the naive first-order bound \eqref{eq:trivial-FO}. Here, we further compare the bound 
\begin{align}
\label{eq:our-second-order}
L(h_\mv) \leq \frac{4(K-1)}{K}\left(\E_{h\sim \Qc}[\testerr(h)] - \frac{1}{2}\E_{h,h'\sim \Qc}[\dis(h,h')]\right) 
\end{align}
to other known results in the literature. The closest in form is a bound specialized to binary case from \cite{second-order-mv-bounds2020}, which gives 
\begin{align}
\label{eq:binary-second-order}
L(h_\mv) \leq 4\E_{h\sim \Qc}[\testerr(h)] - 2\E_{h,h'\sim \Qc}[\dis(h,h')].
\end{align}
Note that plugging in $K=2$ to \eqref{eq:our-second-order}, we obtain the bound $2\E_{h\sim \Qc}[\testerr(h)] - \E_{h,h'\sim \Qc}[\dis(h,h')]$, immediately improving on \eqref{eq:binary-second-order} by a factor of 2 (interestingly, the same factor that we save on the first-order bound). Hence, provided the competence assumption holds, our bound is a direct improvement on this bound, and furthermore generalizes directly to the $K$-class setting.

To our knowledge, the sharpest known upper bound on the majority-vote classifier is the general form of the C-bound given in \cite{LAVIOLETTE201715}, which states, provided $\E[M_\rho(X,Y)] > 0$, 
\begin{align}
\label{eq:app-c-bound}
\testerr(h_\mv) \leq 1 -\frac{\E[M_\rho(X,Y)]^2}{\E[M_\rho^2(X,Y)]},
\end{align}
where $M_\rho(X,Y) = \E_{h\sim \rho}[ \mathbb{1}(h(X) = Y)] - \max_{j\neq Y}\E_{h\sim \rho}[\mathbb{1}(h(X) = j)]$
is called the \emph{margin}. Unfortunately, the use of the margin function makes direct analytical comparison to our bound difficult. However, the bounds can be compared empirically, where the relevant quantities are estimated on hold-out data. In Figure \ref{fig:bound-comparison}, we compare the value of our bound against the value of the multi-class C-bound, on tasks for which we have verified the competence assumption holds. We find that in all but one case (random forests with MNIST), our bound is superior empirically, sometimes significantly. Interestingly, we observe that our bound does particularly well on tasks with only a few classes. This behavior might be attributed to the constant $\frac{4(K-1)}{K}$ in the upper bound \eqref{eq:our-second-order} which increases as the number of classes $K$ grows.

% This is a general limitation of the bound in \eqref{eq:our-second-order}; in the next Section, we discuss this issue in further detail, and show how it can be at least partially overcome. 

\begin{figure}
    \centering
    \includegraphics[scale=0.4]{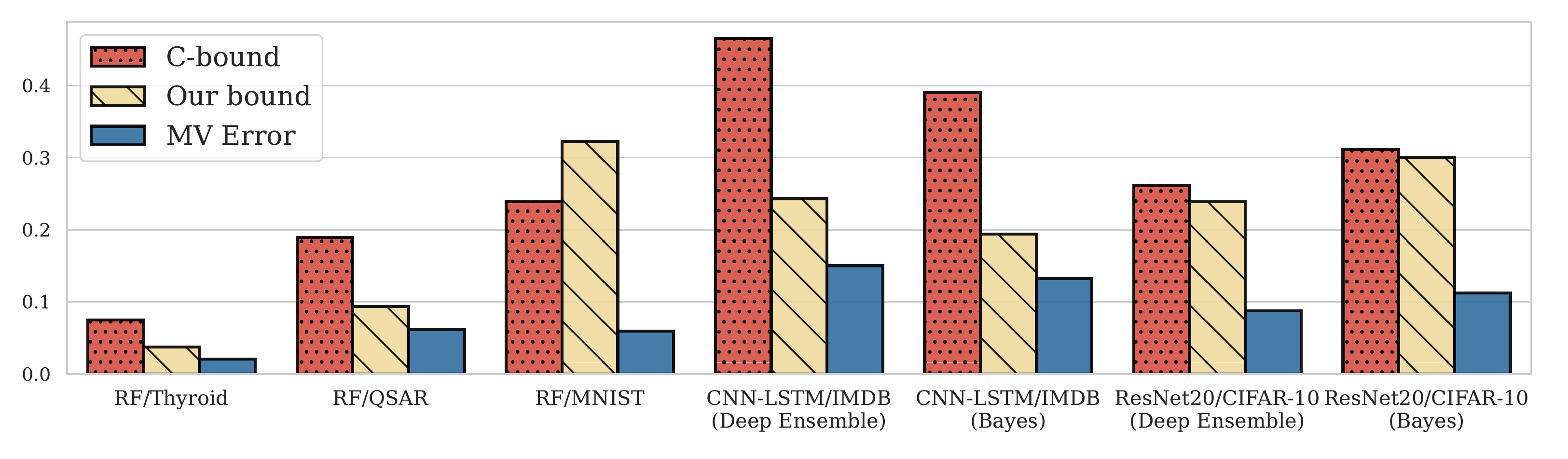}
    \caption{\textbf{Our bound \eqref{eq:our-second-order} versus the multi-class C-bound \eqref{eq:app-c-bound}.}}
    \label{fig:bound-comparison}
\end{figure}

\section{Additional empirical results}
\label{app:additional-empirical}

\subsection{Experimental details}
\label{app:emp-details}
\paragraph{Bagged random feature classifiers.} 
We consider ensembles of random ReLU feature classifiers, constructed as follows. 
For each classifier, we draw a random matrix $\boldsymbol{U} \in \R^{N\times d}$, whose rows $\boldsymbol{u}_j$ are drawn from the uniform distribution on the sphere $\mathbb{S}^{d-1}$. 
For a given input $\xb \in \R^d$, we compute the feature $\zb(\xb) = \sigma(\boldsymbol{Ux})$ where $\sigma(t) = \max(t,0)$ is the ReLU function. 
We then fit a multi-class logistic regression model in \texttt{scikit-learn} \cite{scikit-learn} using these (random) features. 
To form an ensemble of these classifiers, we additionally perform bagging, by sampling a different set of size $n$ with replacement from the training set of size $n$, independently for each individual classifier. 
Thus, each classifier is subject to two different types of randomness: the randomness from the sampling of the feature matrix $\boldsymbol{U}$; and the randomness from the bootstrapping of the training data. For the models shown in the competence plot in Figure \ref{fig:competence}, we use $500$ random features and $M=100$ classifiers.\vspace{-0.03cm}

\paragraph{Random forests.} 
We consider random forest (RF) models as implemented in \texttt{scikit-learn} \cite{scikit-learn}, each made up of 20 individual decision trees. 
We vary the maximum number of leaf nodes in each tree to construct models with varying performance.
For the single-ensemble results presented, we use the default parameters implemented in \texttt{scikit-learn}. 
For the random forests, we use a small version of the MNIST dataset with 5000 randomly selected training examples (500 from each of the 10 classes). 
We also use two binary classification datasets retrieved from the UCI repository \cite{Dua:2019}: the QSAR oral toxicity dataset (7.2k train, 1.8k test examples, 1024 features) \cite{qsar-dataset}; and the Thyroid disease dataset (2.5k train, 633 test examples, 21 features) \cite{thyroid-dataset}. For the models shown in the competence plot in Figure \ref{fig:competence}, we use the default settings of the random forest implementation in \texttt{scikit-learn}.

\vspace{-0.03cm}
\paragraph{Deep ensembles.} 
We consider four different architectures for our deep ensembles. 
% For the large-scale empirical evaluations reported in Figures \ref{fig:EIR_vs_DER_scatter} and \ref{fig:resnet_der_interpolation} (and elaborated on in the Appendix) w
We use a standard ResNet18 models \cite{resnet} trained on the CIFAR-10 dataset \cite{cifar10}, using 100 epochs of SGD with momentum $0.9$, weight decay of $5 \times 10^{-4}$ and a learning rate of $0.1$, while varying the batch size and width hyper-parameters. 
We report results from two variants of this empirical evaluation: one in which we employ learning rate decay (by dropping the learning rate to $0.01$ after 75 epochs); and another in which we disable learning rate decay. 
For each setting, we train $5$ models from independent initialization to form the respective ensembles. 
We also evaluate these models on two out-of-distribution databases: CIFAR-10.1 and CIFAR-10-C \cite{recht2018cifar10.1, hendrycks2019robustness} (the latter is itself comprised of 19 different datasets employing various types of data corruption). 
% For the empirical evaluations presented in Figures \ref{fig:competence}, \ref{fig:bound-comparison}, and \ref{fig:bayes-vs-deep-box}, we train 20 models from independent initialization using a batch size of 128, and a ResNet20 architecture, modified to use the Swish activation function, to match the architecture used in \cite{pmlr-v139-izmailov21a} (although, in contrast to that model, which uses filter response normalization to facilitate Bayesian inference, we still employ the usual batch normalization). 
% We also train a CNN-LSTM model, matching the one used in \cite{pmlr-v139-izmailov21a}, trained for 10 epochs with a batch size of 128 using the Adam optimizer \cite{ADAM}. We train 20 models from independent intializations on the IMDB dataset \cite{imdb-dataset}. 
Finally, we evaluate deep ensembles of 25 standard BERT models \cite{devlin-etal-2019-bert}, provided with the paper \cite{multibert}, fine-tuned on the GLUE classification tasks \cite{wang2018glue}.\vspace{-0.1cm}

\paragraph{Bayesian ensembles.} 
For the Bayesian ensembles used in this paper, we consider samples provided in \cite{pmlr-v139-izmailov21a}, obtained via large-scale sampling from a Bayesian posterior using Hamiltonian Monte Carlo.
To our knowledge, these samples are the most precisely representative of a theoretical Bayesian neural network posterior publicly available.
In particular, we use samples on the CIFAR-10 datasets with a ResNet20 
%\addressed{NO, don't introduce a new name for a well-known thing}
architecture, and the IMDB dataset on the CNN-LSTM architecture. 
We defer to the original paper~\cite{pmlr-v139-izmailov21a} for additional~details. 

\subsection{More competence plots}
In this section, we provide additional empirical results.
%%
%%\michael{Add a sentence of two stating what those are and why we present, e.g., different data/ensembles, but the conclusions are broadly the same.}
%%
%%\subsection{More competence plots}
%%\label{app:more-competence}
%%
%%\michael{Are we going to have an additional subsubsection here?  If not, let's remove it and just put that under the section.}
%%
%%\michael{Can you make the labels for the figures in this section larger, so they can be read, they are too small.}
%%
To further verify that the competence assumption holds broadly in practice, here we include several more examples of competence plots for experiments presented in the main text. \vspace{-0.1cm}

\paragraph{ResNet18 on CIFAR-10 OOD variants.} In Figures \ref{fig:competence-cifar-variants-lr-decay} and \ref{fig:competence-cifar-variants-no-lr-decay}, we plot competence plots for the ResNet18 ensembles on the CIFAR-10, CIFAR-10.1 and a subset of the CIFAR-10-C datasets \cite{recht2018cifar10.1, hendrycks2019robustness}. We find that the competence assumption holds across all examples. \vspace{-0.1cm}

\paragraph{Fine-tuned BERT models.}
In Figure \ref{fig:competence-glue}, we provide competence plots for the BERT/GLUE fine-tuning tasks. For the RTE, CoLA, MNLI, QQP and QNLI tasks, we find that the competence assumption holds. However, we find two examples here where it does not: the MRPC and SST-2 tasks, although the extent to which the assumption is violated in minor. Since these are particularly small datasets, this may also be a product of noise from low sample size.
\vspace{-0.1cm}

\begin{figure}[h!]
    \centering
    \includegraphics[scale=0.28]{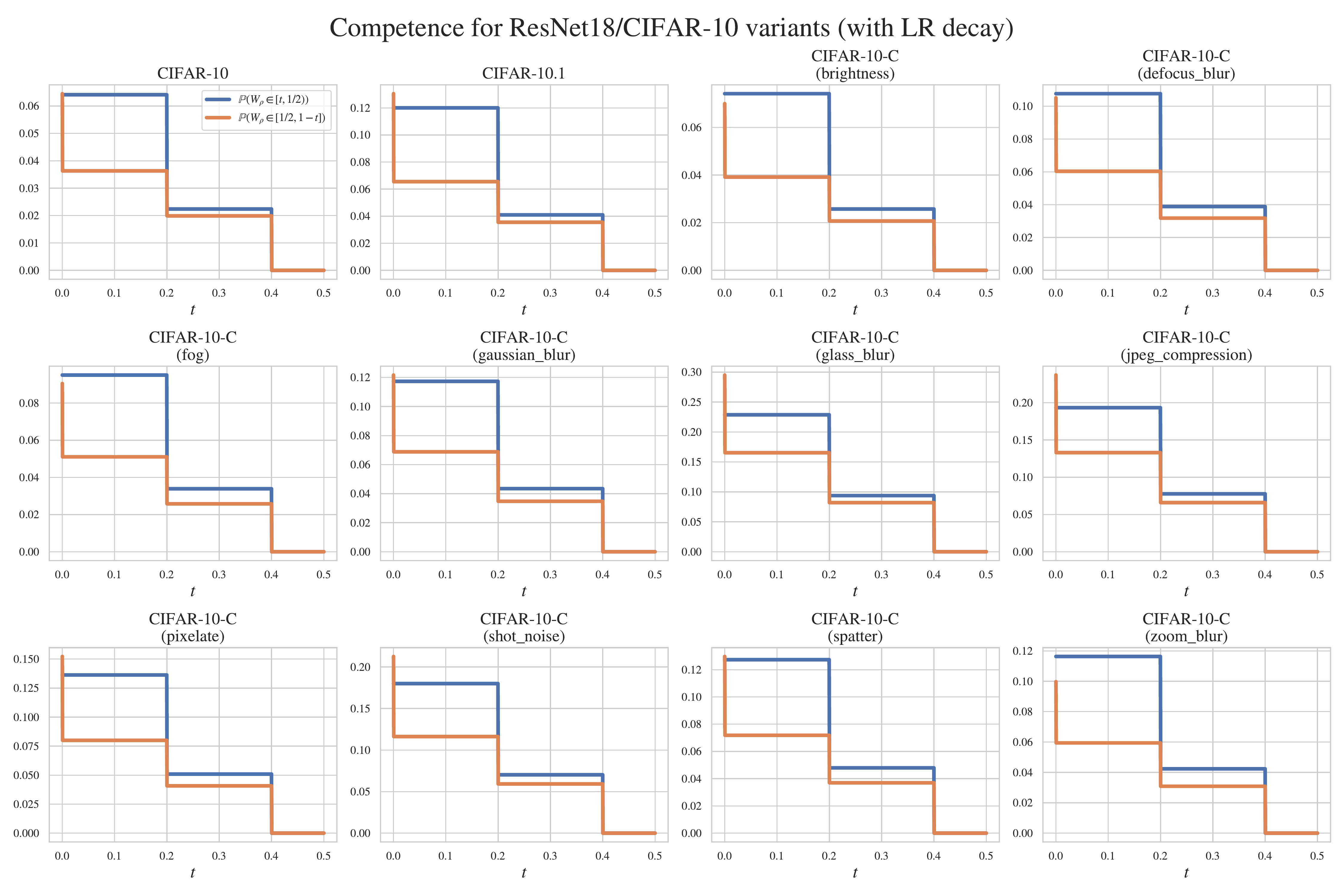}
    \caption{\textbf{Competence for ResNet18/CIFAR-10 variants (models with learning rate decay).} We observe that the competence assumption holds across all tasks.}
    \vspace{-0.4cm}
    \label{fig:competence-cifar-variants-lr-decay}
\end{figure}

\begin{figure}[h!]
    \centering
    \includegraphics[scale=0.28]{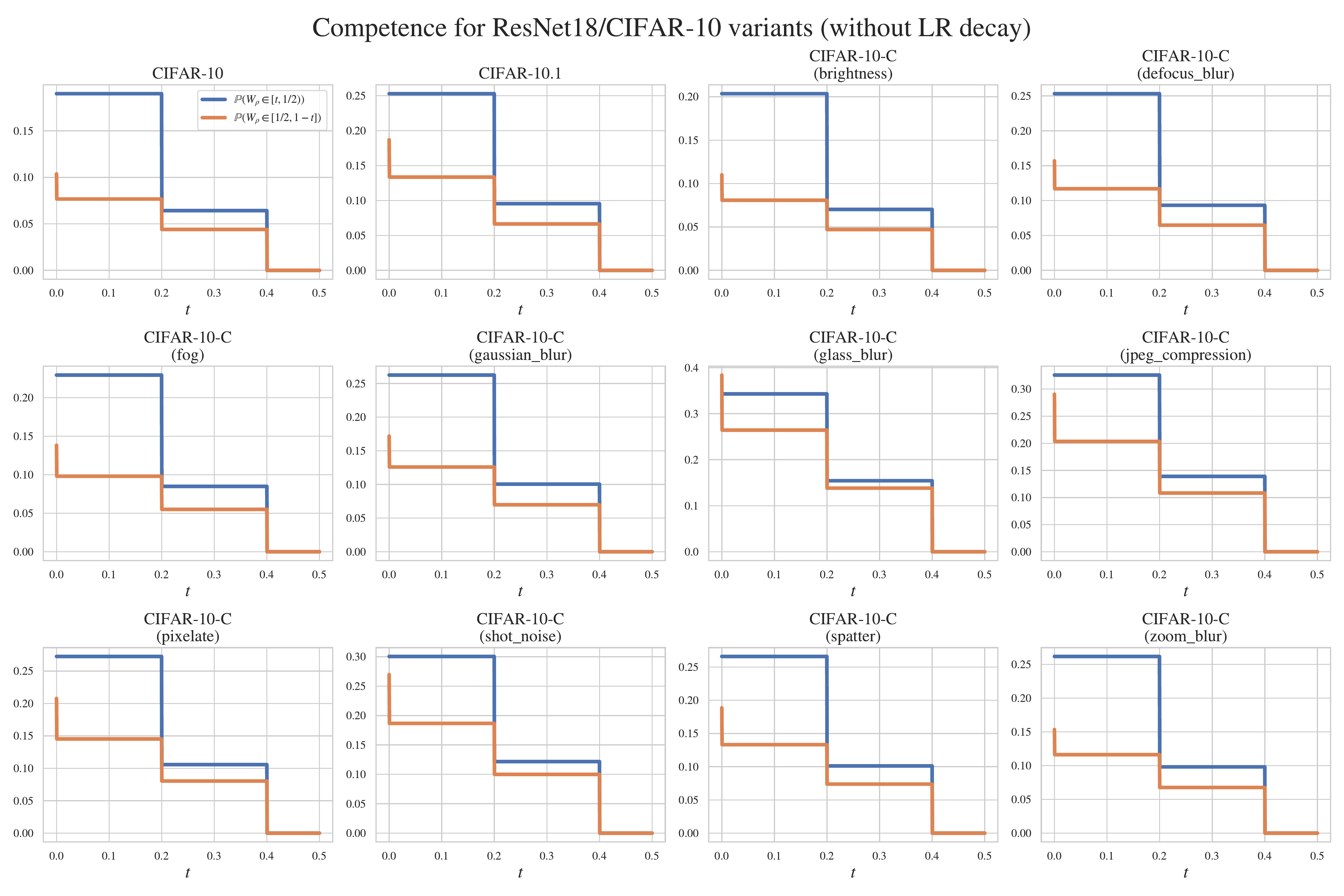}
    \caption{\textbf{Competence for ResNet18/CIFAR-10 variants (models without learning rate decay).} We observe that the competence assumption holds across all tasks.}
    \vspace{-0.4cm}
    \label{fig:competence-cifar-variants-no-lr-decay}
\end{figure}

\begin{figure}[h!]
    \centering
    \includegraphics[scale=0.3]{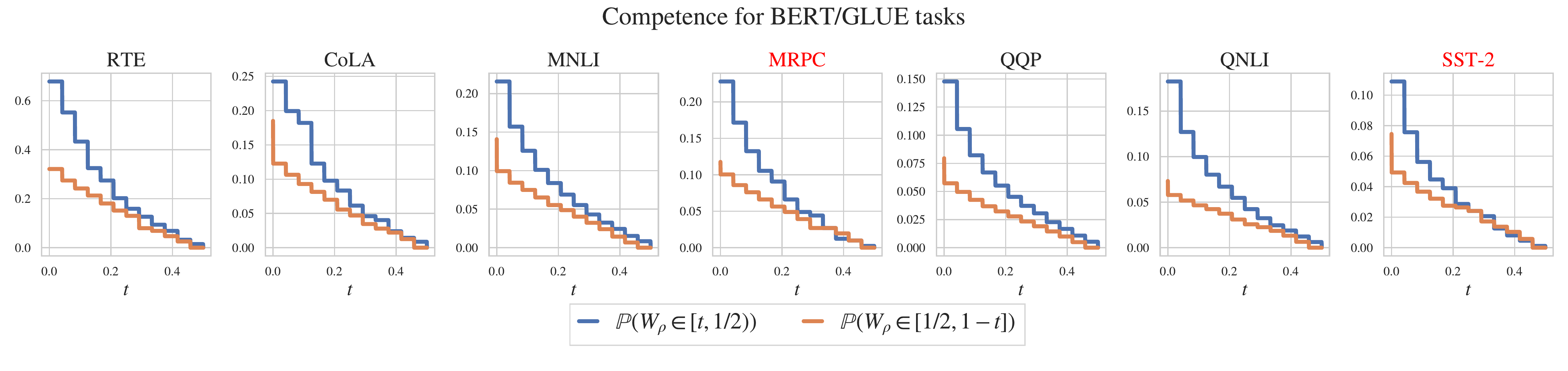}
    \caption{\textbf{Competence for BERT/GLUE fine-tuning tasks.} The competence assumption holds for the RTE, CoLA, MNLI, QQP and QNLI tasks, though interestingly, we find that the competence assumption is (to a small degree) violated for two of the tasks: MRPC and SST-2.}
    \label{fig:competence-glue}
\end{figure}

%%%%%%%%%%%%%%%%%%%%%%%%%%%%%%%%%%%%%%%%%%%%%%%
% Examples
%%%%%%%%%%%%%%%%%%%%%%%%%%%%%%%%%%%%%%%%%%%%%%%
%\newpage
\section{Pathological ensembles satisfying $L(h_\mv) = 2\E[L(h)]$}\label{app:pathological}

In this section, we provide two pathological examples of ensembles that makes the ``first-order'' upper bound tight. 
In particular, the second example shows that positive margin condition, i.e., $\E[M_\rho(X,Y)] > 0$ where $M_\rho(X,Y) = \E_{h\sim \rho}[ \mathbb{1}(h(X) = Y)] - \max_{j\neq Y}\E_{h\sim \rho}[\mathbb{1}(h(X) = j)]$, from existing literature is not enough to rule out pathological cases. 
Recall that the first-order bound introduced in Section \ref{sec:existing-theory} is the following:
\begin{align*}
0 \leq \testerr(h_\mv) \leq \mathbb{P}(W_\rho \geq 1/2) \leq 2\, \mathbb{E}[W_\rho] = 2\,\E_{h\sim \Qc}[\testerr(h)].
\end{align*}

\begin{figure}[h!]
\centering
  \begin{subfigure}{0.48\textwidth}
    \centering
    \includegraphics[scale = 0.41]{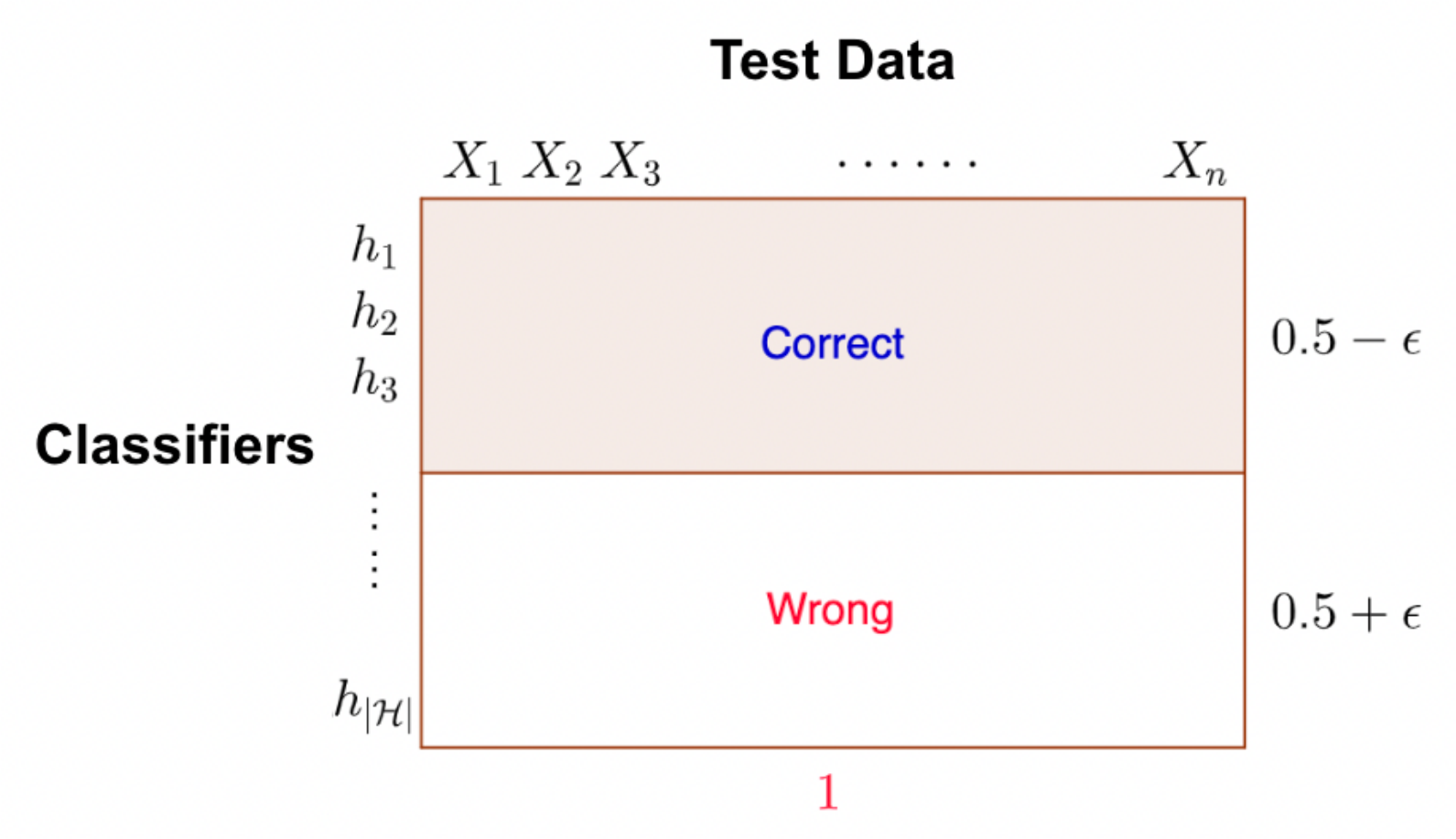}
    \caption{Composition of classifiers in Example \ref{ex:example-1}} \label{fig:firsttight1}
  \end{subfigure}%
  \hfill   % maximize separation between the subfigures
  \begin{subfigure}{0.48\textwidth}
    \centering
    \includegraphics[scale = 0.41]{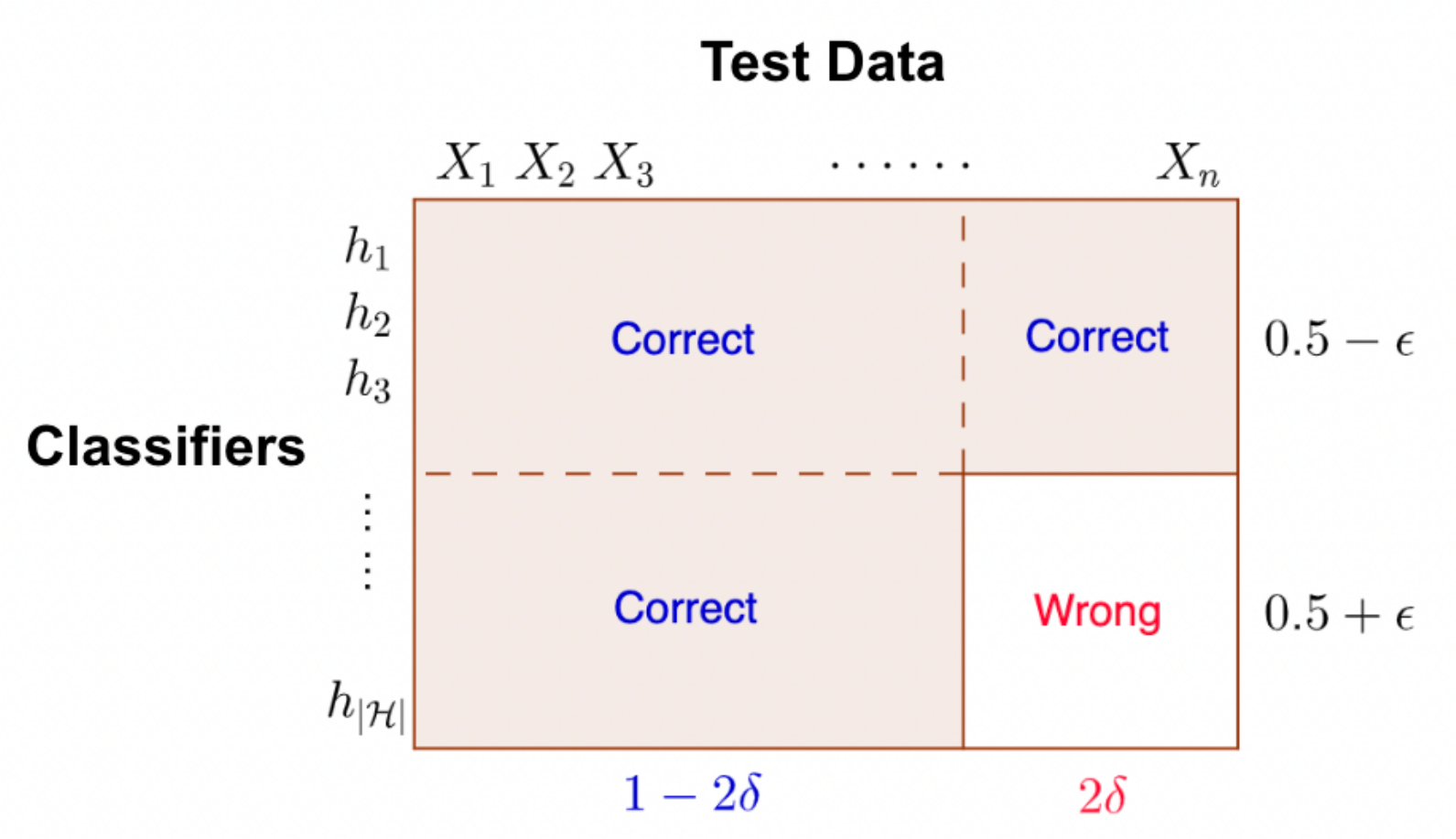}
    \caption{Composition of classifiers in Example \ref{ex:example-2}} \label{fig:firsttight2}
  \end{subfigure}%
\caption{\textbf{Illustration of the composition of classifiers given in Examples~\ref{ex:example-1} and~\ref{ex:example-2}.} On each plot, the area of the white box equals to the average test error rate. On Figure \ref{fig:firsttight1}, the majority vote error rate is $1$, while the average test error rate is $0.5 + \epsilon$. On Figure \ref{fig:firsttight2}, the majority vote error rate is $2\delta$ (Rightmost $2\delta$ test data) while the average test error rate is $\delta(1 + 2\epsilon)$. The margin of each composition of classifiers is $2\epsilon \to 0$ and $1-2\delta(1+2\epsilon) > 0$, respectively.} \label{fig:firsttight}
\end{figure}

\begin{example}[The first-order upper bound is tight]
\label{ex:example-1}
%Some example of when the first-order bound is saturated
Consider a classification problem with two classes. 
For given $\epsilon > 0$, suppose slightly less than half, $0.5-\epsilon$, fraction of classifiers are the perfect classifier, correctly classifying test data with probability 1, and the other $0.5+\epsilon$ fraction of classifiers are completely wrong, incorrectly predicting on test data with probability 1. 
With this composition of classifiers, the average error rate is $0.5+\epsilon$ and the majority vote error rate is $1$. 
Taking $\epsilon \to 0$ concludes that the first-order upper bound \eqref{eq:trivial-FO} is tight. 
A visual illustration of the composition of classifiers is given in Figure \ref{fig:firsttight1}.
\end{example}

The condition $\E[M_\rho(X,Y)] > 0$ rules out the ensemble described in Example \ref{ex:example-1}. 
Nonetheless, the first-order bound $2\E[L(h)]$ is tight \emph{even when} $\E[M_\rho(X,Y)] > 0$ is satisfied, as we show with the following example.

\begin{example}[The first-order upper bound is tight even when the margin is large]
\label{ex:example-2}
We again consider a classification problem with two classes. 
For given $\epsilon > 0$, as in Example \ref{ex:example-1}, slightly less than half, $0.5-\epsilon$, fraction of classifiers are the perfect classifier. 
All of the other $0.5+\epsilon$ fraction of classifiers, on the contrary, now correctly predict on the same $1-2\delta$ fraction of the test data and incorrectly predict on the other $2\delta$ fraction of the test data.
With this composition of classifiers, the majority vote error rate is $2\delta$ even when the average error rate is $\delta(1+2\epsilon)$. 
In addition, unlike the composition of classifiers in Example \ref{ex:example-1}, the margin of which is $2\epsilon$, the margin of the new composition of classifiers is 
%1 \cdot (1-2\delta) + (-2\epsilon) \cdot (2\delta) = 
$1-2\delta(1+2\epsilon)$, which can be any value smaller than $1$. 
Taking $\epsilon \to 0$ concludes that the first-order upper bound \eqref{eq:trivial-FO} is also tight when the margin is arbitrarily high. 
A visual illustration of the composition of classifiers is given in Figure \ref{fig:firsttight2}.
\end{example}

\end{document}